\def\year{2021}\relax
\documentclass[letterpaper]{article} % DO NOT CHANGE THIS
\usepackage{aaai21}  % DO NOT CHANGE THIS
\usepackage{times}  % DO NOT CHANGE THIS
\usepackage{helvet} % DO NOT CHANGE THIS
\usepackage{courier}  % DO NOT CHANGE THIS
\usepackage[hyphens]{url}  % DO NOT CHANGE THIS
\usepackage{graphicx} % DO NOT CHANGE THIS
\usepackage{graphicx}  % DO NOT CHANGE THIS
\usepackage{natbib}  % DO NOT CHANGE THIS AND DO NOT ADD ANY OPTIONS TO IT

\usepackage{caption} % DO NOT CHANGE THIS AND DO NOT ADD ANY OPTIONS TO IT
\usepackage[switch]{lineno}
\usepackage{amsmath}
\usepackage{amsthm}
\usepackage{amssymb}
\usepackage[ruled,vlined,linesnumbered]{algorithm2e}
\newtheorem{theorem}{Theorem}
\theoremstyle{definition}
\newtheorem{definition}{Definition}
\newtheorem{lemma}{Lemma}
\newtheorem{prop}{Proposition}
\DeclareMathOperator*{\argmax}{arg\,max}
\DeclareMathOperator*{\argmin}{arg\,min}

 \usepackage{color} %(if used in text)

\providecommand{\keywords}[1]
{
  \small	
  \textbf{\textit{Keywords---}} #1
}

\title{Interpretable Clustering on Dynamic Graphs \\ with Recurrent Graph Neural Networks}
\author{

    Yuhang Yao, %\textsuperscript{\rm 1}
    Carlee Joe-Wong\\%\textsuperscript{\rm 1}\\
}
\affiliations{
    Carnegie Mellon University\\
    \{yuhangya,cjoewong\}@andrew.cmu.edu

}
\begin{document}
%\linenumbers 
%remove for camera-ready version

\maketitle

\begin{abstract}
We study the problem of clustering nodes in a dynamic graph, where the connections between nodes and nodes' cluster memberships may change over time, e.g., due to community migration. We first propose a dynamic stochastic block model that captures these changes, and a simple decay-based clustering algorithm that clusters nodes based on weighted connections between them, where the weight decreases at a fixed rate over time. This decay rate can then be interpreted as signifying the importance of including historical connection information in the clustering. However, the optimal decay rate may differ for clusters with different rates of turnover. We characterize the optimal decay rate for each cluster and propose a clustering method that achieves almost exact recovery of the true clusters. We then demonstrate the efficacy of our clustering algorithm with optimized decay rates on simulated graph data. Recurrent neural networks (RNNs), a popular algorithm for sequence learning, use a similar decay-based method, and we use this insight to propose two new RNN-GCN (graph convolutional network) architectures for semi-supervised graph clustering. We finally demonstrate that the proposed architectures perform well on real data compared to state-of-the-art graph clustering algorithms.

\end{abstract}
\keywords{Node Classification, Dynamic Graphs, RNN, GNN}

\section{Introduction}

Clustering nodes based on their connections to each other is a common goal of analyzing graphs, with applications ranging from social to biological to logistics networks. Most such clustering approaches assume that the connections (i.e., edges) between nodes, and thus the optimal clusters, do not change over time~\cite{lei2015consistency,qin2013regularized}. In practice, however, many graph structures will evolve over time. Users in social networks, for example, may migrate from one community to another as their interests or employment status changes, forming new connections with other users (i.e., new edges in the graph) and changing the cluster or community to which they belong. Thus, clustering algorithms on such evolving graphs should be able to track changes in cluster membership over time.

A major challenge in tracking cluster membership changes is to carefully handle historical information and assess its value in predicting the current cluster membership. Since clusters will often evolve relatively slowly, an extreme approach that does not consider edges formed in the past risks ignoring useful information about the majority of nodes whose memberships have not changed. On the other hand, making no distinction between historical and more recently formed edges may lead to slow detection of nodes' membership changes, as the historically formed edges would dominate until the nodes have enough time to make connections within their new clusters. Prior works have balanced these effects by introducing a \emph{decay rate}: the weight of each edge is reduced by a constant decay factor in each time step between the connection formation and the time at which cluster membership is estimated. The cluster membership can then be estimated at any given time, by taking the weighted connections as input to a static algorithm like the well-studied spectral clustering~\cite{abbe2017community}.

Accounting for historical node connections with a single decay rate parameter offers the advantage of interpretability: the decay rate quantifies the emphasis put on historically formed edges, which can be tuned for specific datasets. Yet while prior works have examined the optimal decay rate for stylized network models, they use a single decay rate for all edges~\cite{keriven2020sparse}. In practice, the optimal rate will likely vary, e.g., with %low decay rates for nodes belonging to clusters with little membership turnover and 
higher decay rates for clusters with higher membership turnover where historical information might reflect outdated cluster memberships, making it less useful. Introducing different decay rates for each cluster, on the other hand, raises a new challenge: since we do not know the true cluster memberships for each node, we may use the wrong decay rate if a node is erroneously labeled. Moreover, the optimal decay rates for different clusters will be correlated due to connections between nodes in different clusters that themselves must be optimally weighted, potentially making the decay rates difficult to optimize.

More sophisticated semi-supervised clustering methods combine LSTM (long-short-term memory) or RNN (recurrent neural network) structures with graph convolutional networks (GCNs), producing a neural network that classifies nodes based on their cluster membership labels. This network can be carefully trained to optimize the use of historical edge information, without explicitly specifying different node decay rates~\cite{pareja2020evolvegcn}. However, while such algorithms show impressive empirical performance on large graph datasets, they are generally not easily interpretable.
%Through extensive training on specific datasets, these structures can, for example, handle historical connections differently for . However, while semi-supervised clustering via graph neural networks has shown good performance on a variety of real-world datasets, \carlee{cite} neural network structures are notoriously hard to interpret, making it difficult to understand how they utilize historical and inter-cluster information \carlee{cite}.

Our work seeks to connect the theoretical analysis of graph clustering algorithms with the graph neural networks commonly used in practice. Our key insight in doing so is that \emph{prefacing a GCN with a RNN layer} can be interpreted as imposing a decay rate on node connections that depends on each node's current cluster membership, and then approximating spectral clustering on the resulting weighted graph via the GCN. Following this insight, we propose two new \emph{transitional RNN-GCN neural network architectures} (RNNGCN and TRNNGCN) for semi-supervised clustering. We derive the theoretically optimal decay rates for nodes in each cluster under stylized graph models, and show that the weights learned for the RNN layer in TRNNGCN qualitatively match the theoretically optimal ones.
%
%these two lines of work by extending current theoretical analysis of the optimal decay rate to take into account cluster heterogeneity. We then show that semi-supervised clustering under a recurrent neural network (RNN) uses similar cluster-specific decay rates, and that the learned rates are qualitatively similar to the theoreticallly optimal ones.
%
%Clustering has long raised attention worldwide. However, the clustering in evolving networks with dynamic edges and labels has not been well-studied and lack analysis in theorem. Former works considering on change of labels with time by rapid change on static graph structure, which doesn't capture the evolution of graph and progressive change during evolution. 
%
%\carlee{Then, talk about the question of how much historical information to include, and introduce the notion of decay rates to quantify this (cite some prior papers that also do this)}
%
%\carlee{Finally, talk about optimizing the decay rate and the connection to RNNs}
%
%
After reviewing related work on theoretical and empirical graph clustering, we make the following specific contributions:
\begin{itemize}
    \item A \textbf{theoretical analysis of the optimal decay rates} for spectral clustering algorithms applied to the dynamic stochastic block model, a common model of graph clustering dynamics~\cite{keriven2020sparse}.
    \item Two \textbf{new RNN-GCN neural network architectures} that use an interpretable RNN layer to capture the dynamics of evolving graphs and GCN layers to cluster the nodes.
    \item Our algorithm can achieve \textbf{almost exact recovery} by including a RNN layer that decays historical edge information. Static methods can only partially recover the true clusters when nodes change their cluster memberships with probability $\mathcal{O}\left(\frac{\log n}{n}\right)$, $n$ being the number of nodes.
    \item \textbf{Experimental results} on real and simulated datasets that show our proposed RNN-GCN architectures outperform state-of-the-art graph clustering algorithms.
\end{itemize}

% 1. Derived a Dynamic SBM model which capture the transition of community and progressive change during evolution. \carlee{check prior work to see if we should claim this as a contribution?}

% 2. Give theory analysis on classification performance by change of forget rate and the optimal forget rate by spectrum clustering method.

% 3. Proposed a model which use RNN to capture the dynamic of evolving networks and give different weight to nodes in different communities and time steps, then use GCN to classify the community of nodes.

% 4. Compare with the partial recovery by GCN when link probability is O($\frac{1}{n}$), our model can achieve nearly exactly recovery by re-weighting graphs which have very similar parameters by theory analysis, which gives the interpretability of parameters in RNN. 

% 5.Experiments on datasets demonstrate the interpretability of real community migration and performance over various methods.

\section{Related Work}

Over the past few years, there has been significant research dedicated to graph clustering algorithms, motivated by applications such as community detection in social networks. While some works consider theoretical analysis of such clustering algorithms, more recently representation learning algorithms have been proposed that perform well in practice with few theoretical guarantees. We aim to \emph{connect} these approaches by using a theoretical analysis of decay-based dynamic clustering algorithms to design a new neural network-based approach that is easily interpretable.

\textbf{Theoretical analyses.}
Traditional spectral clustering algorithms use the spectrum of the graph adjacency matrix to generate a compact representation of the graph connectivity~\cite{lei2015consistency,qin2013regularized}. A line of work on static clustering algorithms uses the stochastic block model for graph connectivity~\cite{abbe2017community}, which more recent works have extended to a dynamic stochastic block model~\cite{keriven2020sparse,pensky2019spectral}. While these works do not distinguish between clusters with different transition probabilities, earlier models incorporate such heterogeneity~\cite{xu2015stochastic}. Other works use a Bayesian approach~\cite{yang2011detecting}, scoring metrics~\cite{agarwal2018dyperm,zhuang2019dynamo}, or multi-armed bandits~\cite{mandaglio2019dynamic} to detect communities and their evolution, while~\citet{xu2010evolutionary} use a decay rate similar to the one we propose.

As \textbf{representation learning} becomes popular, graph neural networks~\cite{zhang2018link,wu2020comprehensive,kipf2016semi} such as
GraphSage~\cite{hamilton2017inductive} have been used to cluster nodes in graphs based on (static) connections between nodes and node features. Graph Attention Networks (GAT)~\cite{velivckovic2017graph,xu2019self} use attention-based methods to construct a neural network that highlights the relative importance of each feature, while dynamic supervised~\cite{kumar2019predicting} and unsupervised~\cite{goyal2020dyngraph2vec} methods can track general network dynamics, or may be designed for clustering on graphs with dynamic edges and dynamic node features~\cite{chen2018gc,xu2019spatio,xu2020inductive}. EvolveGCN~\cite{pareja2020evolvegcn} usesa  GCN to evolve the RNN weights, which is similar to our approach; however, we ensure interpretability of the RNN weights by placing the RNN before the GCN layers, which we show improves the clustering performance.

Finally, several works have considered the \textbf{interpretability of general GCN and RNN structures}~\cite{dehmamy2019understanding,liang2017interpretable,guo2019exploring}, such as GNNExplainer~\cite{ying2019gnnexplainer}. In the context of graph clustering, some works have used attention mechanisms to provide interpretable weights on node features~\cite{xu2019spatio}, but attention may not capture true feature importance~\cite{serrano2019attention}. Moreover, these works do not consider the importance of \emph{historical} information, as we consider in this work.

% 1.1 traditional method

% Common methods are clustering graphs based on embedding the connections between nodes like Spectral Clustering by getting the spectrum of connective matrix of graphs \cite{lei2015consistency}\cite{qin2013regularized}.

% 1.2
% %whole line of theory analysis of spectral clustering, from static to dynamic

% 1.2.1  static method and theory

% To theoretically study the performance of clustering graphs, Stochastic Block model is commonly used\cite{abbe2017community}

% 1.2.2  dynamic methods

% Some recent progresses have been made on the clustering in the temporal graphs. 

% Some works on clustering dynamic graphs. 
% \cite{yang2011detecting} uses a Bayesian approach to detect communities and their evolutions.
% \cite{xu2010evolutionary} uses forget gate 

% 1.2.3  dynamic theory analysis

% Theory analysis on Dynamic SBM \cite{keriven2020sparse}  \cite{pensky2019spectral}

% Transitional model\cite{xu2015stochastic} Classes have different change probability

%Evolving graph \cite{anagnostopoulos2016community} some query of graphs like online learning

\section{Model}
We first introduce a dynamic version of the Stochastic Block Model (SBM) often used to study graph clustering~\cite{holland1983stochastic,abbe2017community}, which we will use for our theoretical analysis in the rest of the paper.

\subsection{Stochastic Block Model} For positive integers $K$ and $n$, a probability vector $p\in [0,1]^K$, and a symmetric connectivity matrix $B\in[0,1]^{K\times K}$, the SBM defines a random graph with $n$ nodes split into $K$ clusters. The goal of a prediction method for the SBM is to correctly divide nodes into their corresponding clusters, based on the graph structure. Each node is independently and randomly assigned a cluster in $\{1,...,K\}$ according to the distribution $p$; we can then say that a node is a ``member'' of this cluster. Undirected edges are independently created between any pair of nodes in clusters $i$ and $j$ with probability $B_{ij}$,
%Each node only belongs to one class. Not sure the former description is clear enough 
where the $(i,j)$ entry of $B$ is 
\begin{equation}
B_{ij}=\left\{
\begin{aligned}
 \alpha& ,\;i=j\\
\tau \alpha & ,\;i\neq j,
\end{aligned}
\right.
\end{equation}
for $\alpha\in(0,1)$ and $\tau\in(0,1)$, implying that the probability of an edge forming between nodes in the same cluster is $\alpha$ (which is the same for each cluster) and the edge formation probability between nodes in different clusters is $\tau \alpha$.

Let $\Theta \in {\{0,1\}}^{n\times K}$ denotes the matrix representing the nodes' cluster memberships, where $\Theta_{ik}=1$ indicates that node $i$ belongs to the $k$-th cluster, and is $0$ otherwise. 
We use $A\in\{0,1\}^{n \times n}$ to denote the (symmetric) adjacency matrix of the graph, where $A_{ij}$ indicates whether there is a connection (edge) between node $i$ and node $j$. From our node connectivity model, we find that given $\Theta$, for $i<j$, we have
\begin{equation}
    A_{ij}|\{\Theta_{ik}=1,\Theta_{jl}=1\} \backsim \text{Ber}(B_{kl}),
\end{equation} 
where $\text{Ber}(p)$ indicates a Bernoulli random variable with parameter $p$. We define $A_{ii}=0$ (nodes are not connected directly to themselves) and since all edges are undirected, $A_{ij}=A_{ji}$. We further define the connection probability matrix $P=\Theta B \Theta^T \in [0,1]^{n\times n}$, where $P_{ij}$ is the connection probability of node $i$ and node $j$ and
%\begin{equation}
    $\mathbb{E}[A]=P-\text{diag}(P)$.
%\end{equation}

\subsection{Dynamic Stochastic Block Model} We now extend the SBM model to include how the graph evolves over time. We consider a set of discrete time steps $t = 1,2,\ldots,T$.
%For positive integers $T$ which denotes the total time steps, similarly, 
At each time step $t$, the Dynamic SBM generates new intra- and inter-cluster edges according to the probabilities $\alpha$ and $\tau\alpha$ as defined for the SBM above. All edges persist over time. %(not persist)
We assume a constant number of nodes $n$, number of clusters $K$, and connectivity matrix $B$, but the node membership matrix $\Theta_t$ depends on time $t$, i.e., nodes' cluster memberships may change over time. We similarly define the connectivity matrix $P_t = \Theta_t B (\Theta_t)^T$.

We model changes in nodes' cluster memberships as a Markov process with a constant transition probability matrix $H\in[0,1]^{K\times K}$. Let $\varepsilon_j\in(0,1)$ denotes the change probability of nodes in cluster $j$, i.e., the probability a node in cluster $j$ changes its membership. At each time step, node $v_i$ in cluster $j$ changes its membership to cluster $k$ with the following probability (independently from other nodes):
\begin{equation*}
    H_{j,k}=\mathbb{P}\left[\Theta_{ik}^{t}=1|\Theta_{ij}^{t-1}=1\right]=\left\{
\begin{aligned}
1-\varepsilon_j & , & j=k\\
\frac{\varepsilon_{j}}{K-1} & , & j\neq k,
\end{aligned}
\right.
\end{equation*}
Note that $\varepsilon_j$ may be specific to cluster $j$, e.g., if some clusters experience less membership turnover. We give an example of such a graph in our experimental evaluation.
The goal of a clustering algorithm on a graph is to recover the membership matrix $\Theta$ up to column permutation. Static clustering algorithms give an estimate $\hat{\Theta}$ of the node membership; a dynamic clustering algorithm should produce such an estimate for each time $t$. We define two performance metrics for these estimates (in dynamic graphs, they may be evaluated for an estimate $\hat{\Theta} = \hat{\Theta}_t$ relative to $\Theta = \Theta_t$ at any time $t$):

\begin{definition}
[Relative error of $\hat{\Theta}$]
The relative error of a clustering estimate $\hat{\Theta}$ is
\begin{equation}
E(\hat{\Theta},\Theta)=\frac{1}{n} \min_{\pi\in \mathcal{P}} \|\hat{\Theta} \pi-\Theta\|_0,    
\end{equation}
where $\mathcal{P}$ is the set of all $K\times K$ permutation matrices and $\|.\|_0$ counts the number of non-zero elements of a matrix.
\end{definition}

\begin{definition}
[Almost Exact Recovery] 
A clustering estimate $\hat{\Theta}$ achieves almost exact recovery when
\begin{equation}
 \mathbb{P}\left[1- \frac{1}{n} \min_{\pi\in \mathcal{P}} \|\hat{\Theta} \pi-\Theta\|_0=1-o(1)\right]=1-o(1).
\end{equation}
which also implies that the expectation of $E(\hat{\Theta},\Theta)$ is $o(1)$.
\end{definition}

Our goal is then to find an algorithm that produces an estimate $\hat{\Theta}$ minimizing $E(\hat{\Theta},\Theta)$. In the next section, we discuss the well-known (static) spectral clustering algorithm and analyze a simple decay-based method that allows a static algorithm to make dynamic membership estimates.

\section{Spectral Clustering with Decay Rates}
We now introduce the Spectral Clustering algorithm and optimize the decay rates to minimize its relative error. %...
\subsection{Spectral Clustering Algorithm}

Spectral Clustering is a commonly used unsupervised method for graph clustering. The key idea is to apply $K$-means clustering to the $K$-leading left singular vectors of the adjacency matrix $A$~\cite{stella2003multiclass}; we denote the corresponding matrix of singular vectors as $E_K$. We then estimate the membership matrix $\bar{\Theta}$ by solving 
\begin{equation}
    (\bar{\Theta},\bar{C})\in \argmin_{\Theta\in \{0,1\}^{n\times K}, C\in\mathbb{R}^{K\times K}} \|\Theta C - E_K\|_F^2,
    \label{eq:kmeans}
\end{equation}
where $\|.\|_F$ denotes the Frobenius norm. It is well known that finding a global minimizer of Eq.~\eqref{eq:kmeans} is NP-hard. However, efficient algorithms~\cite{kumar2004simple} can find a $(1+\delta)$-approximate solution $(\hat{\Theta},\hat{C})$, i.e., with %$K$-means that solves 
%
%\begin{equation*}
    $\|\hat{\Theta}\hat{C}-E_K\|_F^2\leq (1+\delta) \|\bar{\Theta}\bar{C}-E_K\|_F^2$.
%\end{equation*}

\subsection{Introducing Decay Rates}
{In the dynamic SBM, the adjacency matrix $A$ includes edges formed from the initial time step $1$ to the current time step $T$. Let $A_t$ denotes the adjacency matrix only including edges formed at time step $t$. We have $A = \sum_{t=1}^T A_t$}.

Spectral clustering performs poorly on the dynamic SBM:
\begin{prop}[Partial Recovery of Spectral Clustering]\label{prop:recover_rate}
When nodes change their cluster membership over time with probabilities $\varepsilon_j = \mathcal{O}(\frac{\log n}{n})$, {by using the adjacency matrix $A$}, Spectral Clustering recovers the true clusters at time $T$ with relative error $E(\hat{\Theta}_T,\Theta_T) = \mathcal{O}(\frac{\log n}{n} T)$. %\carlee{It's not clear which adjacency matrix is used here in the clustering--$A$ or $A_T$? (actually $A_t$ isn't formally defined)}\yuhang{Change from $\Theta$ to $\Theta_T$ ($\Theta_T$ is formally defined before)} %Add proof in support material
\end{prop}

 To improve the performance, {one can use an exponentially smoothed version $\hat{A}_t$ as input for clustering:}
\begin{equation}
\hat{A}_t=(1-\lambda)\hat{A}_{t-1}+\lambda A_t
\end{equation}
where $\hat{A}_1 = A_1$ and we call $\lambda\in[0,1]$ the \emph{decay rate}~\cite{chi2009evolutionary}. Intuitively, a larger value of $\lambda$ puts less weight on the past information, ``forgetting'' it faster. However, in the dynamic SBM, each cluster $j$ may have a different change probability $\varepsilon_j$, implying that they may benefit from using different decay rates $\lambda$. %which is common in reality and use single value $\lambda$ is not enough to represent the structure. 
We thus introduce a decay matrix $\Lambda\in[0,1]^{K\times K}$ that gives a different decay rate to connections between each pair of clusters:
\begin{equation}
    \hat{A}_t=(1-\Theta_t \Lambda (\Theta_t)^T)\odot \hat{A}_{t-1}+\Theta_t \Lambda (\Theta_t)^T \odot A_t.
\end{equation}

\subsection{Bounding the Relative Error}
Our analysis uses \citet{lei2015consistency}'s result that the relative error rate of the Spectral Clustering on the dynamic SBM at each time $t$ is bounded by the concentration of the adjacency matrix around its expectation:
\begin{equation}
    E(\hat{\Theta},\Theta)\lesssim (1+\delta)\frac{n_{\max}' K}{n\alpha^2 n_{\min}^2 \tau^2}\|\hat{A}-P\|^2,
    \label{eq:errorbound}
\end{equation}
where $n_{\max}'$ and $n_{\min}$ are respectively the second largest and smallest cluster sizes, and $\|.\|$ denotes the spectral norm. 

Thus, $E(\hat{\Theta},\Theta)$ is determined by the concentration $\|\hat{A}-P\|$, where {$\hat{A}=\hat{A}_t$} and $P=P_t=\Theta_t B (\Theta_t)^T$ as defined in the SBM model. To bound this concentration, we consider $K$ diagonal blocks of the adjacency matrix $\hat{A}_t$, with each block corresponding to edges between nodes in a single cluster, after re-indexing the nodes as necessary. % can be separate into $K$ blocks in diagonal 
Let $\hat{A}_{t}^{k}$ denote the block matrix corresponding to cluster $k$, and similarly consider $K$ blocks $P^t_k$ of the connection probability matrix $P_t$. We can then upper-bound $\left\|\hat{A}_{t}^{k}-P_{t}^{k}\right\|$ in terms of the decay rate:
\begin{prop}[Optimal Decay Rate]\label{prop:opt}
The concentration of each block $k$ is upper-bounded by
\begin{equation}
    \left\|\hat{A}_{t}^{k}-P_{t}^{k}\right\|\lesssim E_1(\beta_{k})+E_2(\beta_{k}),
\end{equation}
where $\beta_k$ denotes the maximum decay rate of class $k$ and
\begin{equation} %\left\{
%\begin{aligned}
E_1(\beta_{k}) = \sqrt{ n \alpha \beta_{k}},\;
E_2(\beta_{k}) = \alpha \sqrt{\frac{n^2 \varepsilon_k}{\beta_{k}}},
%\end{aligned}
%\right.
\end{equation}
which is minimized when $\beta_{k}=\sqrt{ n \alpha \varepsilon_k}$.
\end{prop}
We formally prove this result in our supplementary material. The intuition is that if the change probability $\varepsilon_k$ is larger, we need a higher decay rate to remember less past information.
We thus define the decay rates as
\begin{equation}
    \Lambda_{jk}=\left\{
\begin{aligned}
\min(1,\sqrt{n \alpha  \varepsilon_k})& , &j=k\\
1& , &j\neq k.
\end{aligned}
\right.
\end{equation}
This decay rate yields almost exact recovery:
\begin{prop}[Almost Exact Recovery]\label{prop:bound}
 Let $\lambda_{\max}$ denote the maximum element on the diagonal of $\Lambda$. With probability at least $1-n^{-\nu}$ for any $\nu > 0$, at any time $t$ we have 
\begin{equation}
    \left\|\hat{A}_t-P_t\right\|\lesssim \sqrt{ n \alpha \lambda_{\max}}
\end{equation}
When $K$ is constant, $\varepsilon_k=\mathcal{O}\left(\frac{\log n}{n}\right)$ and $\alpha=\mathcal{O}\left(\frac{\log n}{n}\right)$, the relative error is $\mathcal{O}\left(\frac{1}{n^{\frac{1}{4}}\log n}\right)$, which implies almost exact recovery at time $T$. 
\end{prop}

%We prove this result in the supplementary material.

% \begin{equation*}
%     \beta_{jk}=\sqrt{\alpha_n n_k \varepsilon_k}=\sqrt{\alpha_n \frac{\prod_{j=1,j\neq k}^K \varepsilon_j}{\sum_{k=1}^K (\prod_{j=1,j\neq k}^K \varepsilon_j)} n \varepsilon_k}
% \end{equation*}

%j!=k set as 1 remember the information with less large value
\subsection{Connection between GCN and Spectral Clustering}
We empirically demonstrate that Proposition~\ref{prop:opt}'s decay rate is optimal by varying the decay rates used in both spectral clustering and the commonly used Graph Convolutional Network (GCN), which is a first-order approximation of spectral convolutions on graphs~\cite{kipf2016semi}. A multi-layer GCN has the layer-wise propagation rule:
\begin{equation}
    H^{(l+1)}=\sigma(\widetilde{D}^{-\frac{1}{2}}\widetilde{A}\widetilde{D}^{-\frac{1}{2}}H^{(l)}W^{(l)}),
\end{equation}
where $\widetilde{A}=A+I_N$, $I_N$ is the identity matrix, $\widetilde{D}_{ii}=\sum_j \widetilde{A}_{ij}$ and $W^{(l)}$ is a layer-specific trainable weight matrix. The activation function is $\sigma$, typically ReLU (rectified linear units), with a softmax in the last layer for graph clustering. The node embedding matrix in the $l$-th layer is $H^{(l)}\in \mathbb{R}^{N\times D}$, which contains high-level representations of the graph nodes transformed from the initial features; $H^{(0)}=I_N$.  %Based on~\citet{kipf2016semi}'s analysis, the GCN is a first-order approximation of spectral convolutions on graphs. 
%In order to further analysis the correlation and the affect of weight decay rate, we compared the accuracy of Spectral Clustering and GCN. 
\begin{figure}[htp]
    \centering
    \includegraphics[width=0.23\textwidth,trim={0.25cm 0.3cm 1.2cm 0.55cm}, clip]{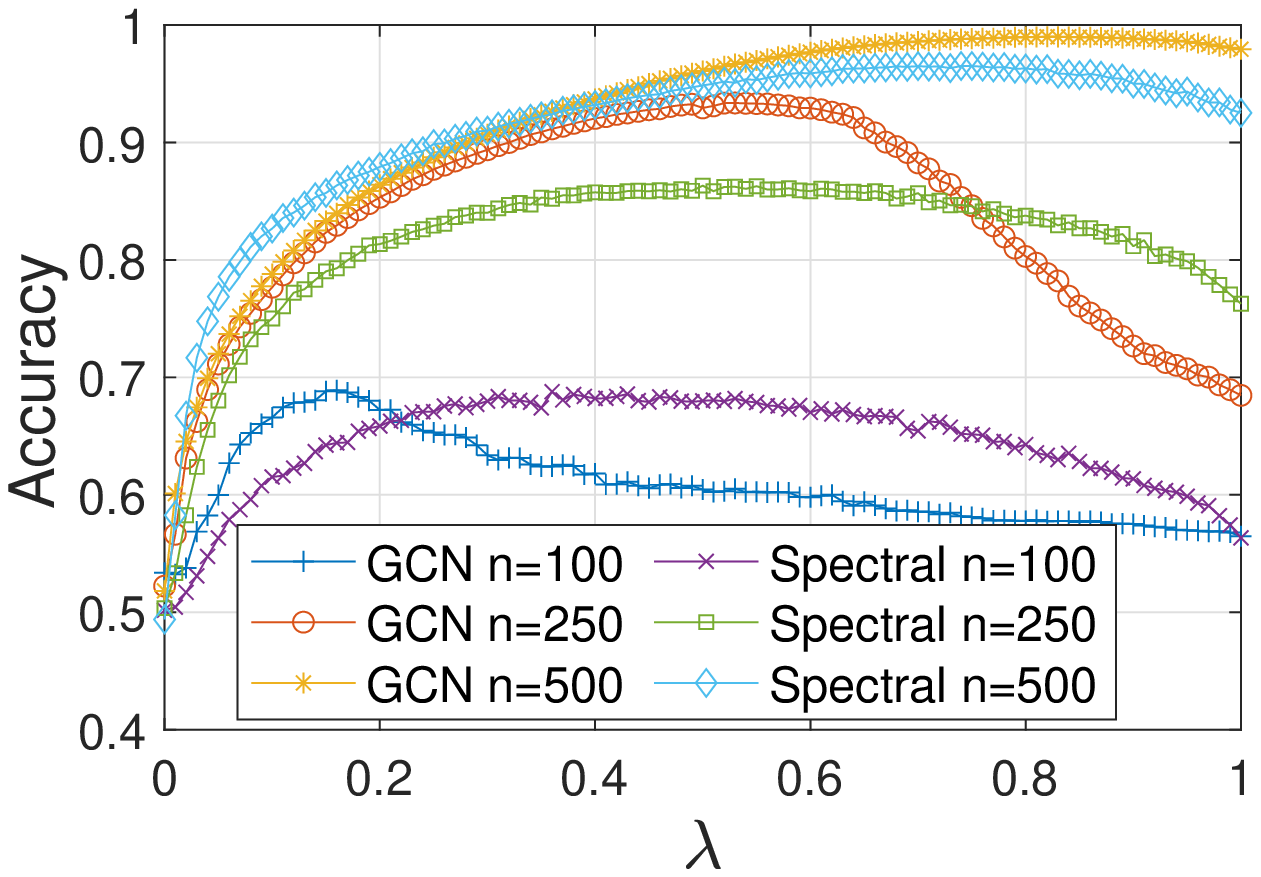}
    \includegraphics[width=0.23\textwidth,trim={0.25cm 0.3cm 1.2cm 0.55cm}, clip]{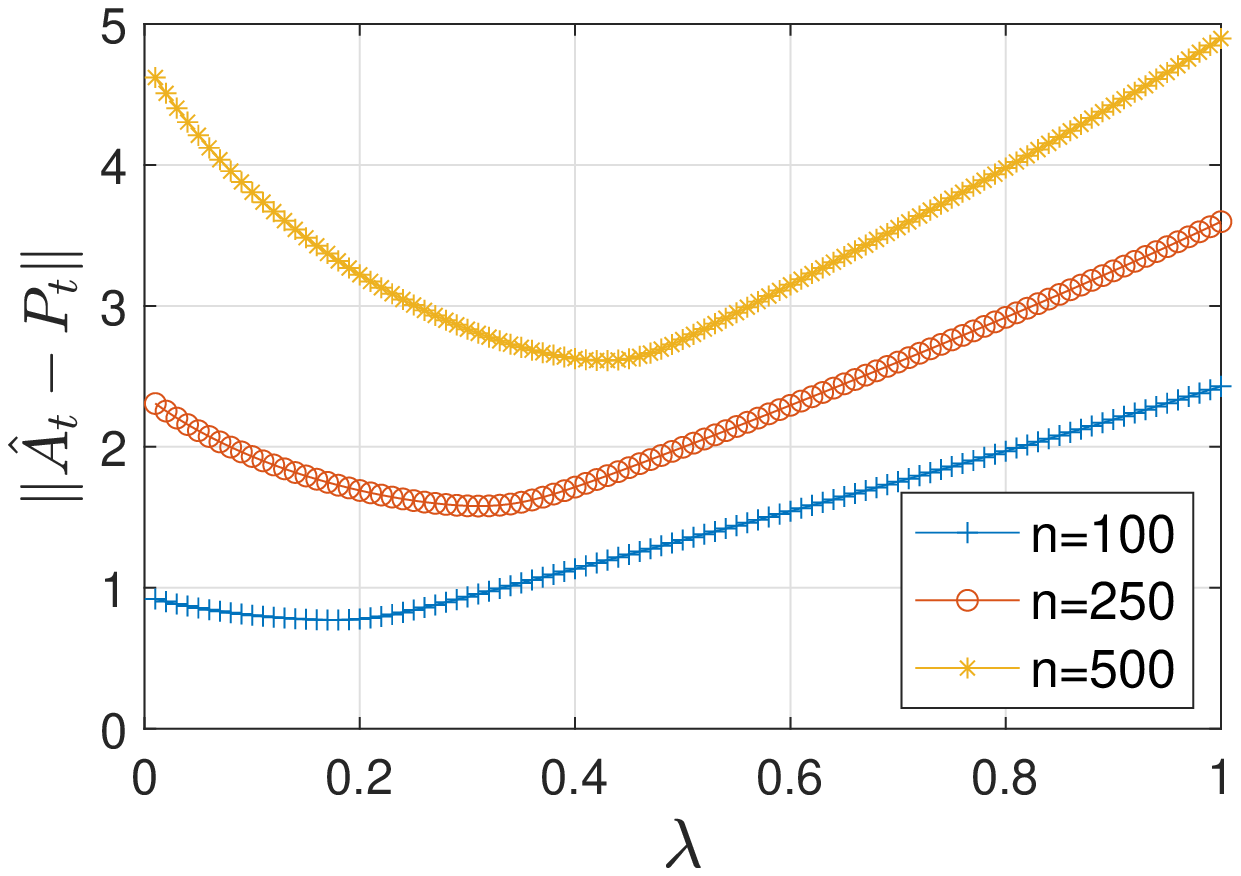}
    \caption{Accuracy and Spectral Norm as we vary $n$. The optimal decay rate $\lambda$ increases with $n$, as in Proposition~\ref{prop:opt}.}
    \label{fig:spec_norm}
\end{figure}
\begin{figure}[ht]
    \centering
    \includegraphics[width=0.23\textwidth,trim={0cm 0.1cm 0cm 0cm}, clip]{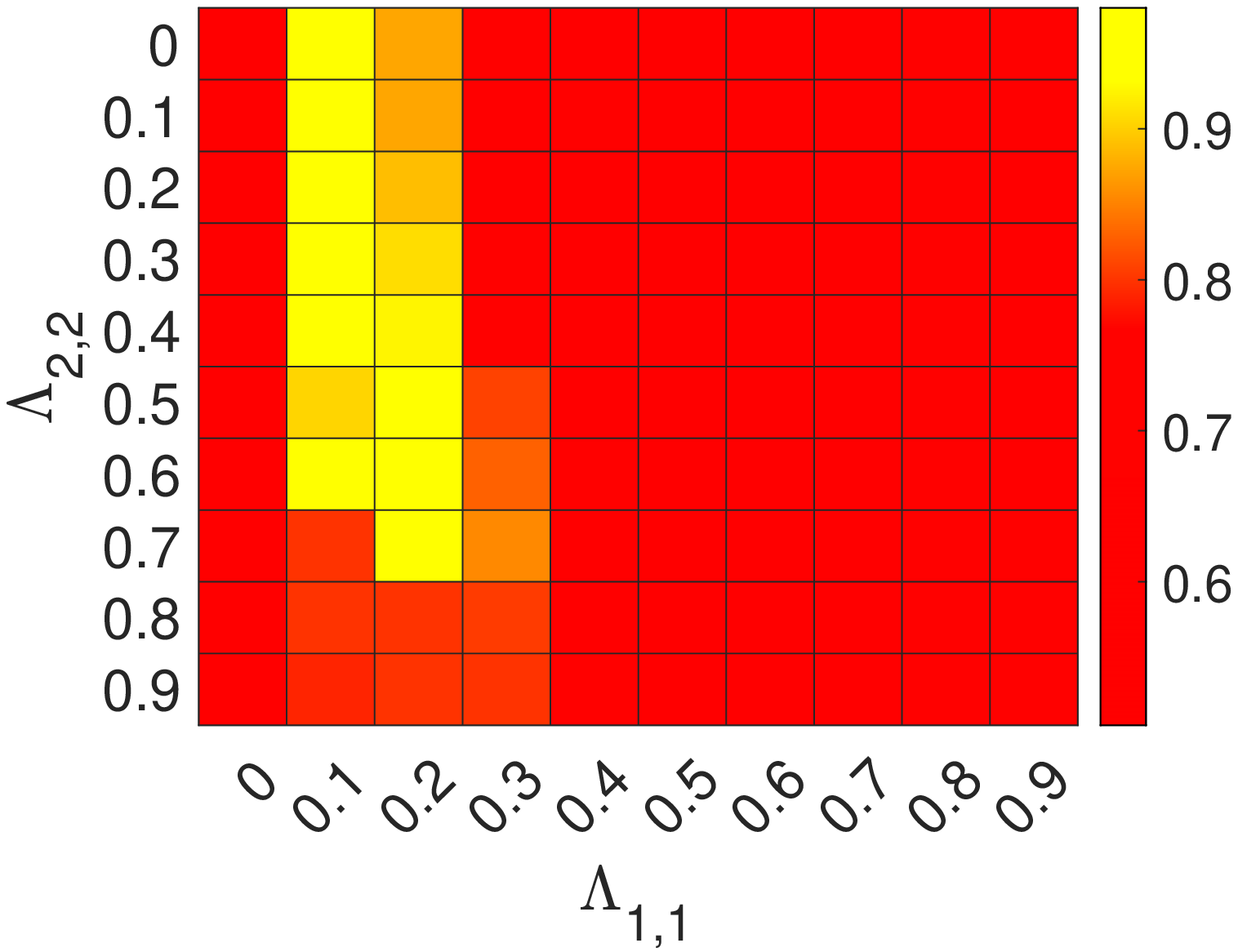}
    \includegraphics[width=0.23\textwidth,trim={0cm 0.1cm 0cm 0cm}, clip]{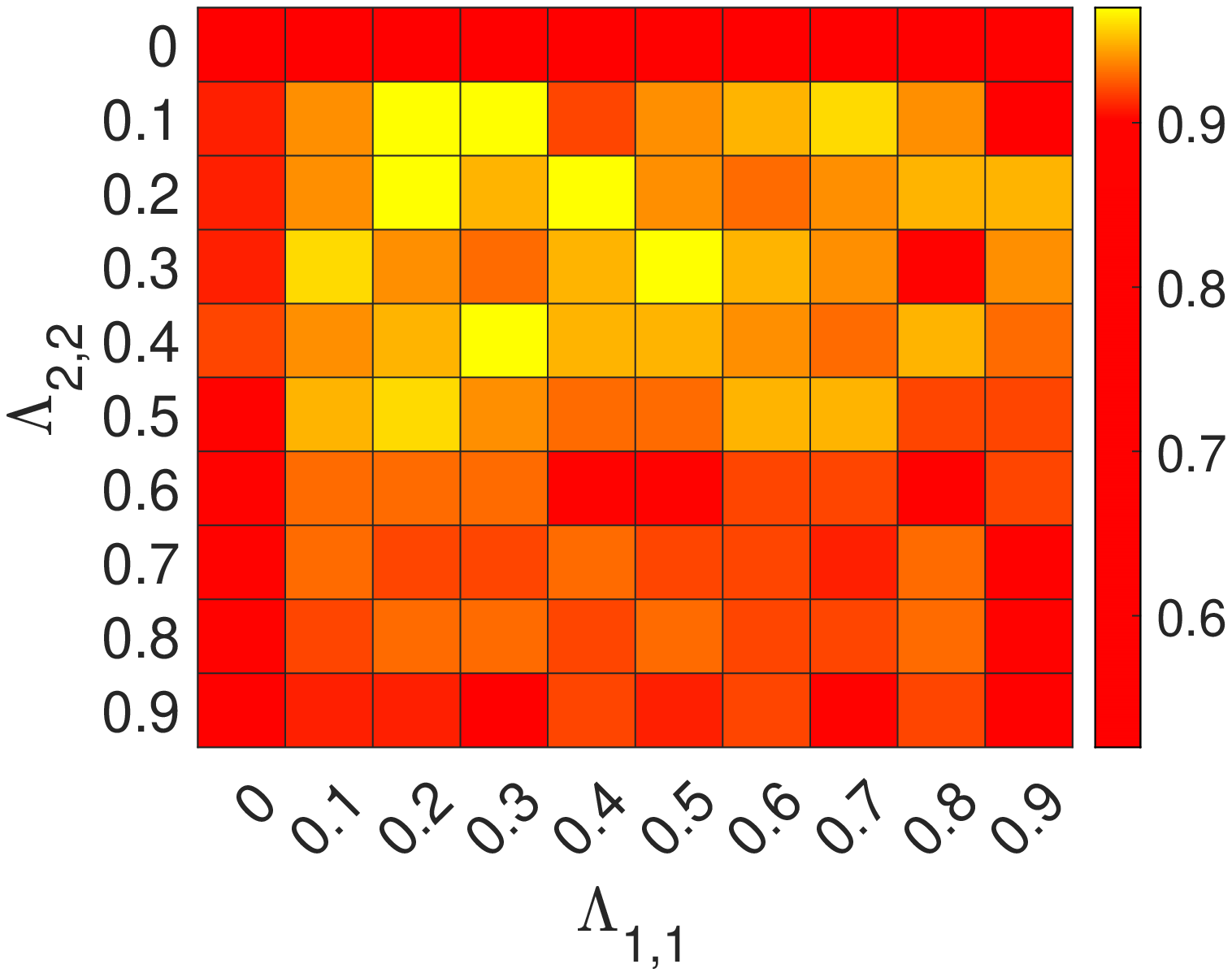}
    \caption{Accuracy as we vary $\Lambda_{1,1}$ and $\Lambda_{2,2}$. Spectral Clustering(left) and GCN(right) have similar optimal decay matrix. The change probabilities are $\varepsilon_1=0.05, \varepsilon_2=0.1$.}
    \label{fig:heat_map}
\end{figure}
Figure \ref{fig:spec_norm} shows that Spectral Clustering and GCN have qualitatively similar accuracy on simulated data as we vary the decay rate $\lambda$ (the same $\lambda$ is used for all nodes). As expected from Eq.~\eqref{eq:errorbound} and Proposition~\ref{prop:opt}, the optimal decay rate $\lambda$ increases as we increase the number of nodes $n$, as does the value of $\lambda$ that minimizes the spectral norm $\|\hat{A} - P\|$. Although the optimal decay rate is consistently larger than the one minimizing the spectral norm (which upper-bounds the relative error as in Eq.~\eqref{eq:errorbound}), GCN's accuracy is more correlated with the spectral norm, which is the first singular value of the smoothed adjacency matrix.

We then perform a grid search for the optimal decay matrix $\Lambda$ on simulated data with $n = 200$ nodes and change probabilities $\varepsilon_1 = 0.05$ and $\varepsilon_2 = 0.1$. As shown in Figure \ref{fig:heat_map}, GCN and Spectral Clustering achieve high accuracy. Cluster 2, which has a higher $\varepsilon_2$, has larger decay rate $\Lambda_{2,2}$, as expected from Proposition~\ref{prop:opt}, for both GCN and Spectral Clustering.
%follow similar patterns and can 

\section{Decay Rates as RNNs}
Although searching for the optimal decay matrix in Spectral Clustering can result in good performance, this method is expensive: the grid search for the optimal decay matrix can be time-consuming, and the time complexity of calculating the spectral norm is $\mathcal{O}(n^3)$. %, which makes Spectral Clustering impractical when the number of nodes becomes large.
In this section, we propose two neural network architectures, RNNGCN and TRNNGCN, that use a single decay rate $\lambda$ and decay matrix $\Lambda$, respectively, and then show they perform well on simulated data.

%we define a neural network approximation to spectral clustering with decay, and show it performs well on simulated data.
%\subsection{Recurrent Graph Neural Network}
%To solve the shortage of spectral clustering with smooth, we proposed Recurrent Graph Neural Network. W

\subsubsection{RNNGCN}

\begin{algorithm}[htp]
\SetAlgoLined
\SetAlgoLined
\SetKwInOut{Input}{Input}
\SetKwInOut{Output}{Output}
\Input{Temporal graph $(A_1,A_2,...,A_T)$,\\ Membership matrix of training data $\Theta^{train}_T$}
\Output{Membership matrix estimate $\hat{\Theta}_T$}

 $\hat{A}=A_0$, $H_0=I_N$\;
\For{\text{iteration} $i=1,...,I$}{

\For{$ t=2,...,T$ }{$\hat{A}_t=(1-\lambda)\hat{A}_{t-1}+\lambda {A}_t$}

$H^{(1)}=\sigma_1(\hat{A}_TH^{(0)}W^{(1)})$\\
$H^{(2)}=\sigma_2(\hat{A}_TH^{(1)}W^{(2)})$\\
CrossEntropyLoss($H^{train}$, ${\Theta}_T^{train}$)\\
Backward()
}
$\hat{\Theta}_T=\text{Onehot}(\argmax_{1\leq j \leq n}H_{jk}^{(2)})$
 \caption{RNNGCN}
 \label{algo:RNNGCN}
\end{algorithm}
The RNNGCN model uses a single decay rate $\lambda\in[0,1]$ as the RNN parameter. RNNGCN first uses a Recurrent Neural Network to learn the decay rate, then uses a two-layer GCN to cluster the weighted graphs.
The formal model is shown in Algorithm \ref{algo:RNNGCN}, where $\sigma_1$ denotes a ReLU layer and $\sigma_2$ is a Softmax layer.

%plot

\subsubsection{Transitional RNNGCN (TRNNGCN)}
%plot

The TRNNGCN network is similar to the RNNGCN, but uses a matrix $\Lambda\in[0,1]^{K\times K}$ to learn the decay rates for different pairs of classes. During the training process, the labels (cluster memberships) of the training nodes are known while the labels of other nodes remain unknown, so we use the cluster prediction $\hat{\Theta}_{i-1}$ from each iteration $i-1$ to determine the decay rates for each node in iteration $i$. 
The TRNNGCN model replaces the decay method (line 4 of Algorithm 1) with
\begin{equation*}
    \hat{A}_t=(1-\hat{\Theta}_{i-1} \Lambda (\hat{\Theta}_{i-1})^T) \odot \hat{A}_{t-1}+\hat{\Theta}_{i-1} \Lambda (\hat{\Theta}_{i-1})^T \odot \hat{A}_t, %%%%%
\end{equation*}
where $\odot$ denotes element-wise multiplication. After each iteration, it calculates $\hat{\Theta}_i$ as the input of the next iteration. 

% \begin{algorithm}[!h]
% \SetAlgoLined
% \SetKwInOut{Input}{Input}
% \SetKwInOut{Output}{Output}
% \Input{Temporal graph $(A_1,A_2,...,A_T)$,\\ Membership matrix of training data $\Theta_T^{train}$}
% \Output{Membership matrix estimate $\hat{\Theta}_T$}

%  $\hat{A}=A_0$, $H_0=I_N$\;
% \For{\text{iteration} $i=1,...,I$}{
% \For{$ t=2,...,T$ }{
% $\hat{A}_t=(1-\Theta_{i-1}^T \Lambda \Theta_{i-1}) \odot \hat{A}_{t-1}+\Theta_{i-1}^T \Lambda \Theta_{i-1} \hat{A}_t.$}
% $H^{(1)}=\sigma_1(\hat{A}_TH^{(0)}W^{(1)})$\\
% $H^{(2)}=\sigma_2(\hat{A}_TH^{(1)}W^{(2)})$\\
% $\hat{\Theta}_{i}=\text{Onehot}(\argmax_{1\leq j \leq n}H_{jk}^{(2)})$\\
% CrossEntropyLoss($H_i^{train}$, $\hat{\Theta}_i^{train}$)\\
% Backward()\\
% }
% $\hat{\Theta}_T=\text{Onehot}(\argmax_{1\leq j \leq n}H_{jk}^{(2)})$

% \caption{TRNNGCN}
% \label{algo:TRNNGCN}
% \end{algorithm}

\subsubsection{Empirical Validation}
We validate the performance of RNNGCN and TRNNGCN on data generated by the dynamic stochastic block model. Our graph has 200 nodes, 23190 edges, 50 time steps and 2 clusters. The probabilities of forming an edge between two nodes of the same or different clusters are $\alpha = 0.02$ and $\tau\alpha = 0.001$, respectively, and a node changes its cluster membership with probability $\varepsilon_1 = 0.05$ and $\varepsilon_2 = 0.1$ for clusters 1 and 2 respectively.

\begin{figure*}[ht]
    \centering
    \includegraphics[width=0.27\textwidth,trim={0.3cm 0 0.8cm 0.3cm}, clip]{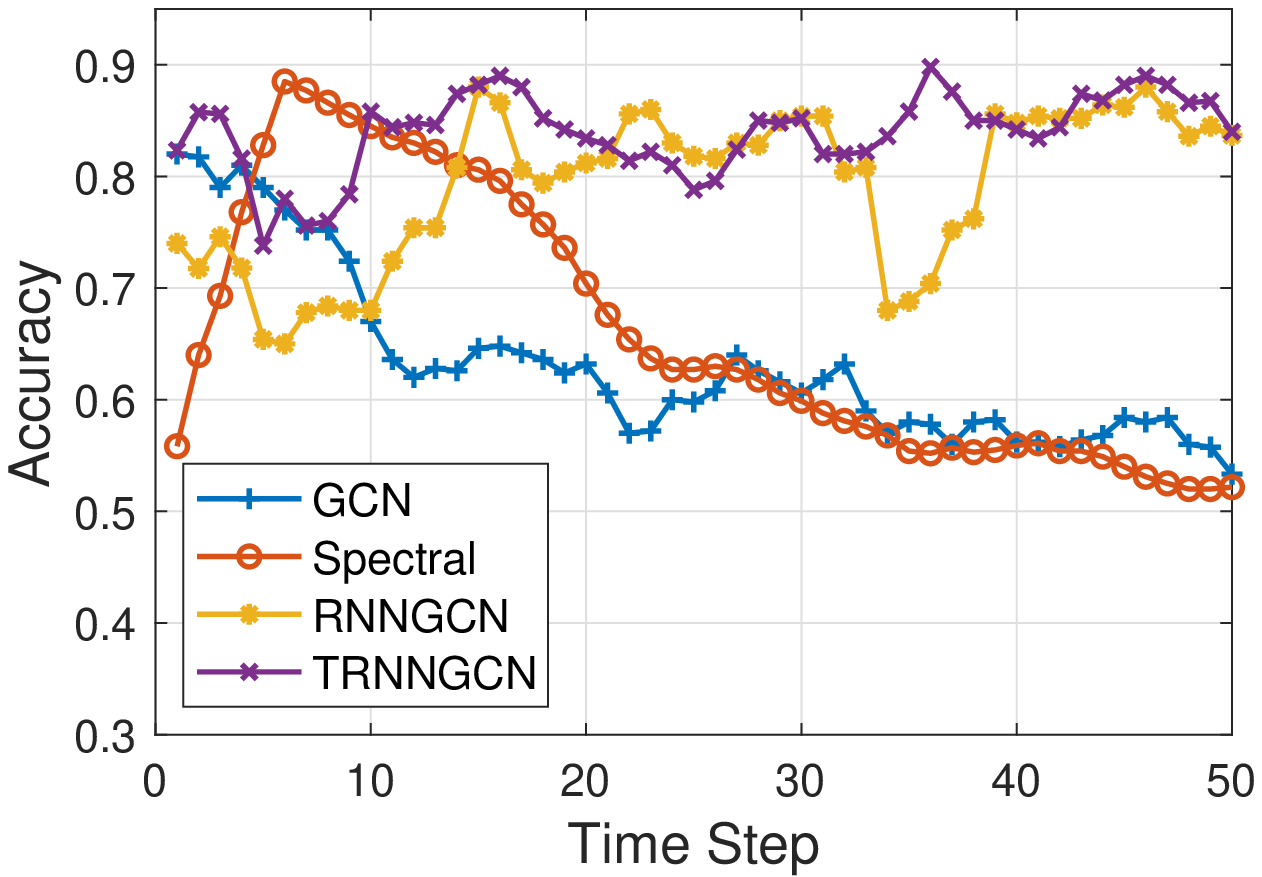}
    \includegraphics[width=0.27\textwidth,trim={0.3cm 0 0.8cm 0.3cm}, clip]{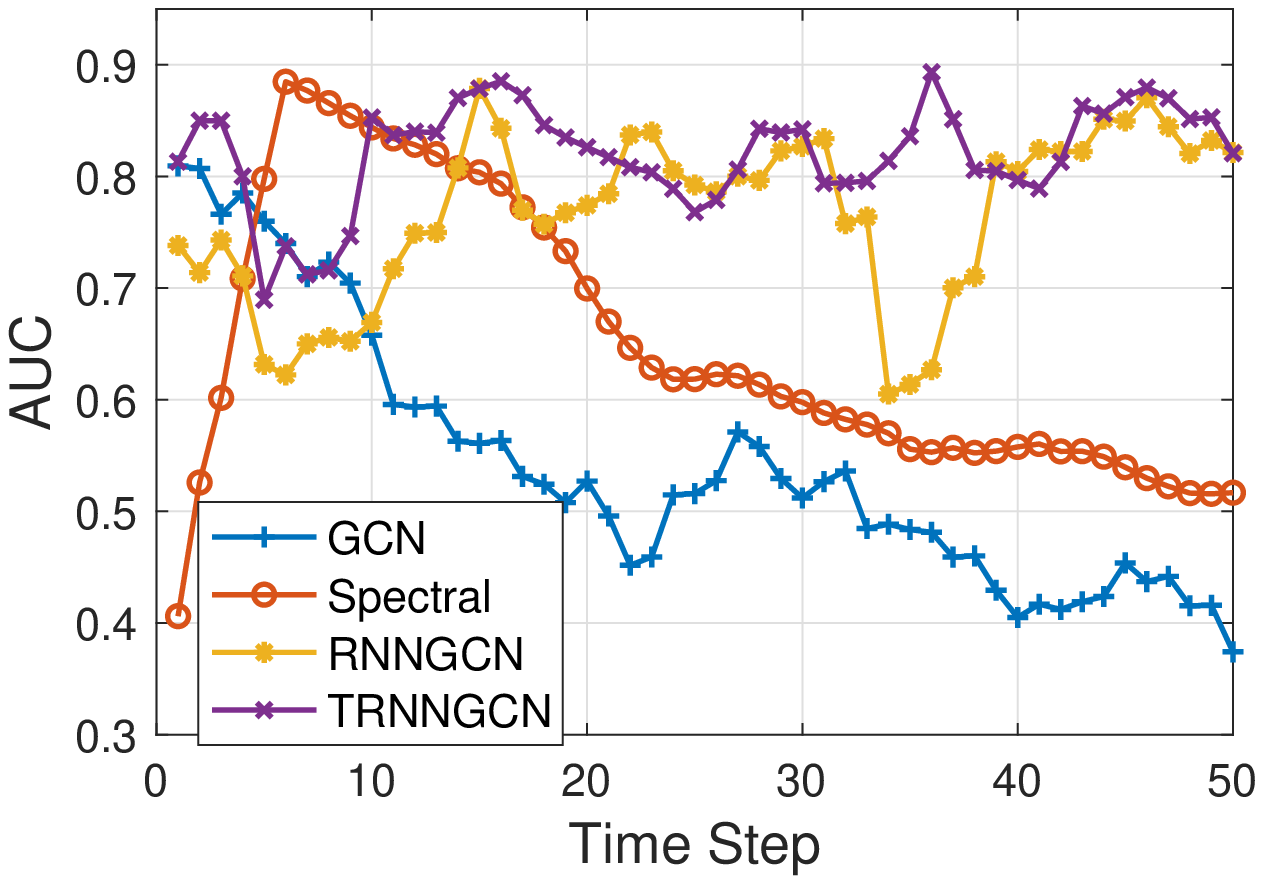}
    \includegraphics[width=0.27\textwidth,trim={0.3cm 0 0.8cm 0.3cm}, clip]{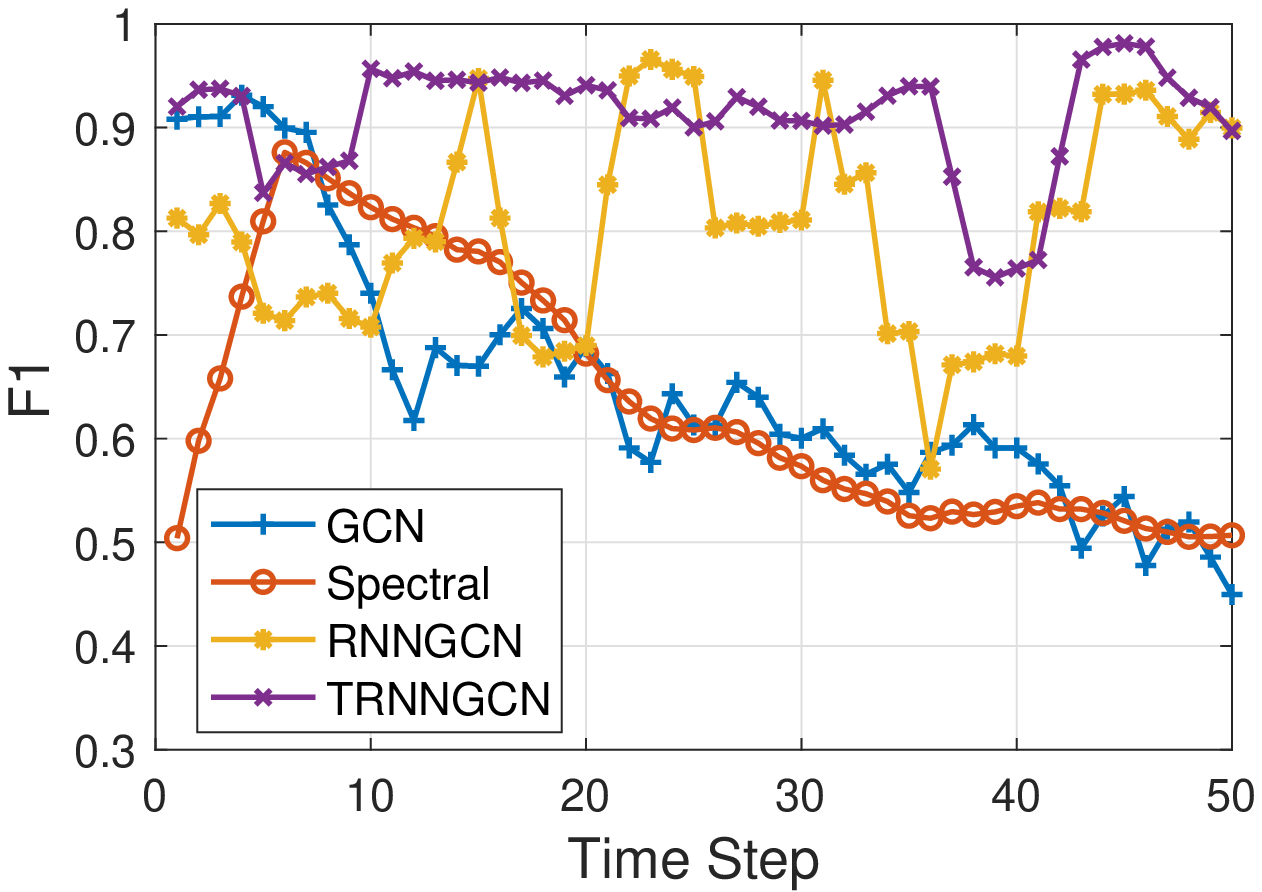}
    \caption{On simulated data with heterogeneous cluster transition probabilities, TRNNGCN and RNNGCN, which use optimized decay rates to account for historical information, outperform the static GCN and Spectral Clustering methods. TRNNGCN slightly outperforms RNNGCN.}
    \label{fig:simulate}
\end{figure*}

Figure \ref{fig:simulate} compares the RNNGCN and TRNNGCN performance with the static Spectral Clustering and GCN methods. For better visualization, the value at each time step is averaged with the 2 timesteps immediately before and after. The performance of GCN and Spectral Clustering decreases over time, as in later timesteps they use accumulated historical information that may no longer be relevant. RNNGCN and TRNNGCN show consistently high performance over time, indicating that they optimally utilize historical information. %The bar plot compares the average performance of the nine methods. need to be included in supplementary material
On average over time, TRNNGCN leads to 5\% accuracy and AUC (area under the ROC curve) improvement, and a 10\% higher F1-score, than RNNGCN, due to using a lower decay rate for the class with smaller change probability.

\section{Experiments}

In this section, we validate the performance of RNNGCN and TRNNGCN on real datasets, compared to state-of-the-art baselines. We first describe the datasets used and the baselines considered, and then present our results.

\subsection{Datasets}
We conducted experiments on five real datasets, as shown in Table \ref{tab:datasets}, which have the properties shown in Table \ref{tab:dataattr}. All datasets have edges that form at different times, although only nodes in DBLP-E change their class (cluster membership) over time. %Though RNNGCN and TRNNGCN are not designed to use historical information of node features, 
We include four datasets with separate, time-varying features associated with each node (DBLP-3, DBLP-5, Brain and Reddit) to test RNNGCN's and TRNNGCN's ability to generalize to datasets with node features.

\begin{table}[h]
    \centering
    
    \begin{tabular}{ccccc}
    \hline
         Dataset& Nodes &Edges & Time Steps &Classes \\
    \hline
          DBLP-E & 6942 & 327392 & 14 &2 \\
          DBLP-3 & 4257 & 23540 & 10 &3\\
          DBLP-5 & 6606 & 42815 & 10 & 5\\
          Brain & 5000 & 1955488 & 12 &10\\
          Reddit & 8291 & 264050 & 10 &4\\
    \hline
    \end{tabular}

    \caption{Real datasets used to evaluate our methods.}
    \label{tab:datasets}
\end{table}

\begin{table}[h]
    \centering
    \begin{tabular}{ccccc}
    \hline
         Dataset&  Dynamic Edge & Dynamic Class & Features\\
    \hline
          DBLP-E &  $\surd$ & $\surd$ &$\times$\\
          DBLP-3 &  $\surd$ & $\times$ &100\\
          DBLP-5 &  $\surd$ & $\times$ & 100\\
          Brain &  $\surd$ & $\times$ &20\\
          Reddit &  $\surd$ & $\times$ &20\\
    \hline
    \end{tabular}
    \caption{Properties of datasets in Table~\ref{tab:datasets}.}
    \label{tab:dataattr}
\end{table}

\subsubsection{DBLP-E} dataset is extracted from the computer science bibliography website DBLP\footnote{https://dblp.org}, which provides open bibliographic information on major computer science journals and conferences. Nodes represent authors, and edges represent co-authorship from 2004 to 2018. Each year is equivalent to one timestep, and co-author edges are added in the year a coauthored paper is published. Labels represent the author research area (``computer networks'' or ``machine learning'') and may change as authors switch their research focus. %It is the co-author graphs where nodes represent authors and the co-author relationship dynamically change each time step (year). The authors are clustered into two classes, Machine Learning and Computer Network, with change of classes at each time step.

\subsubsection{DBLP-3 \& DBLP-5} use the same node and edge definitions as DBLP-E, but also include node features extracted by \texttt{word2vec}~\cite{mikolov2013efficient} from the authors' paper titles and abstracts. The authors in DBLP-3 and DBLP-5 are clustered into three and five classes (research areas) respectively that do not change over time.

\subsubsection{Reddit} dataset is generated from Reddit\footnote{https://www.reddit.com/}, a social news aggregation and discussion website. The nodes represent posts and two posts are connected if they share keywords. Node features are generated by \texttt{word2vec} on the post comments~\cite{hamilton2017inductive}.

\subsubsection{Brain} dataset is generated from functional magnetic resonance imaging (fMRI) data\footnote{https://tinyurl.com/y4hhw8ro}. Nodes represent cubes of brain tissue, and two nodes are connected if they show similar degrees of activation during the time period. Node features are generated by principal component analysis on the fMRI.

\subsection{Baselines and Metrics}
We compare our RNNGCN and TRNNGCN with multiple baselines. GCN, GAT~\cite{velivckovic2017graph} and GraphSage~\cite{hamilton2017inductive} are supervised methods that include node features, while Spectral Clustering is unsupervised without features; all of these methods ignore temporal information. DynAERNN~\cite{goyal2020dyngraph2vec} is an unsupervised method, and GCNLSTM~\cite{chen2018gc} and EGCN~\cite{pareja2020evolvegcn} are supervised methods, which all utilize temporal information of both graphs and features. We evaluate the performance of methods with the standard accuracy (ACC), area under the ROC curve (AUC) and F1-score classification metrics.

\subsection{Experiment Settings}
We divide each dataset into 70\% training/ 20\% validation/ 10\% test points. Each method uses two hidden Graph Neural Network layers (GCN, GAT, GraphSage, etc.) with the layer size equal to the number of classes in the dataset. We add a dropout layer between the two layers with dropout rate $0.5$. We use the Adam optimizer with learning rate $0.0025$. Each method is trained with $500$ iterations. 

For static methods (GCN, GAT, GraphSage and Spectral Clustering) we first accumulate the adjacency matrices of graphs at each time step, then cluster on the normalized accumulated matrix. DynAERNN, GCNLSTM, and EGCN use the temporal graphs and temporal node features as input. For our RNNGCN and TRNNGCN, we use the temporal graphs and the node features at the last time step as input. The code of all methods and datasets are publicly available\footnote{https://github.com/InterpretableClustering/InterpretableClustering}. % make a annoymous account later
\subsection{Experimental Results}
\begin{figure}[ht]
    \centering
    \includegraphics[width=0.44\textwidth]{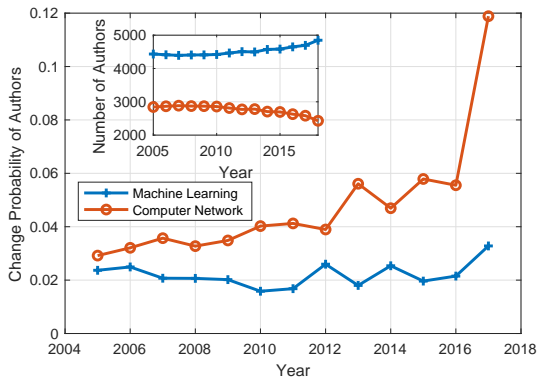}
    \caption{The two classes in DBLP-E exhibit different change probabilities, with computer network authors more likely to change their labels to machine learning. This trend accelerates after 2014.}
    \label{fig:dblpe_num}
\end{figure}
\subsubsection{Node Classification with Temporal Labels}
We first compare the predictions of temporally changing labels in DBLP-E. Figure \ref{fig:dblpe_num} shows the number of authors in the Machine Learning and Computer Network fields in the years 2004-2018, as well as the probabilities that authors in each class change their labels. We observe that (i) the classes have different change probabilities (with users more likely to move from computer networks to machine learning) and (ii) the change probabilities evolve over time, with more users migrating to machine learning since 2013. This dataset thus allows us to test RNNGCN's and TRNNGCN's abilities to adapt the optimal decay rate for each class.
%As the machine learning shows powerful performance on many works, which appeals many authors to have more papers related with machine learning. The change probability of Computer Network increase sharply in recent years. 

\begin{figure*}[ht]
    \centering
    \includegraphics[width=0.23\textwidth,trim={0.25cm 0cm 0.8cm 0.5cm}, clip]{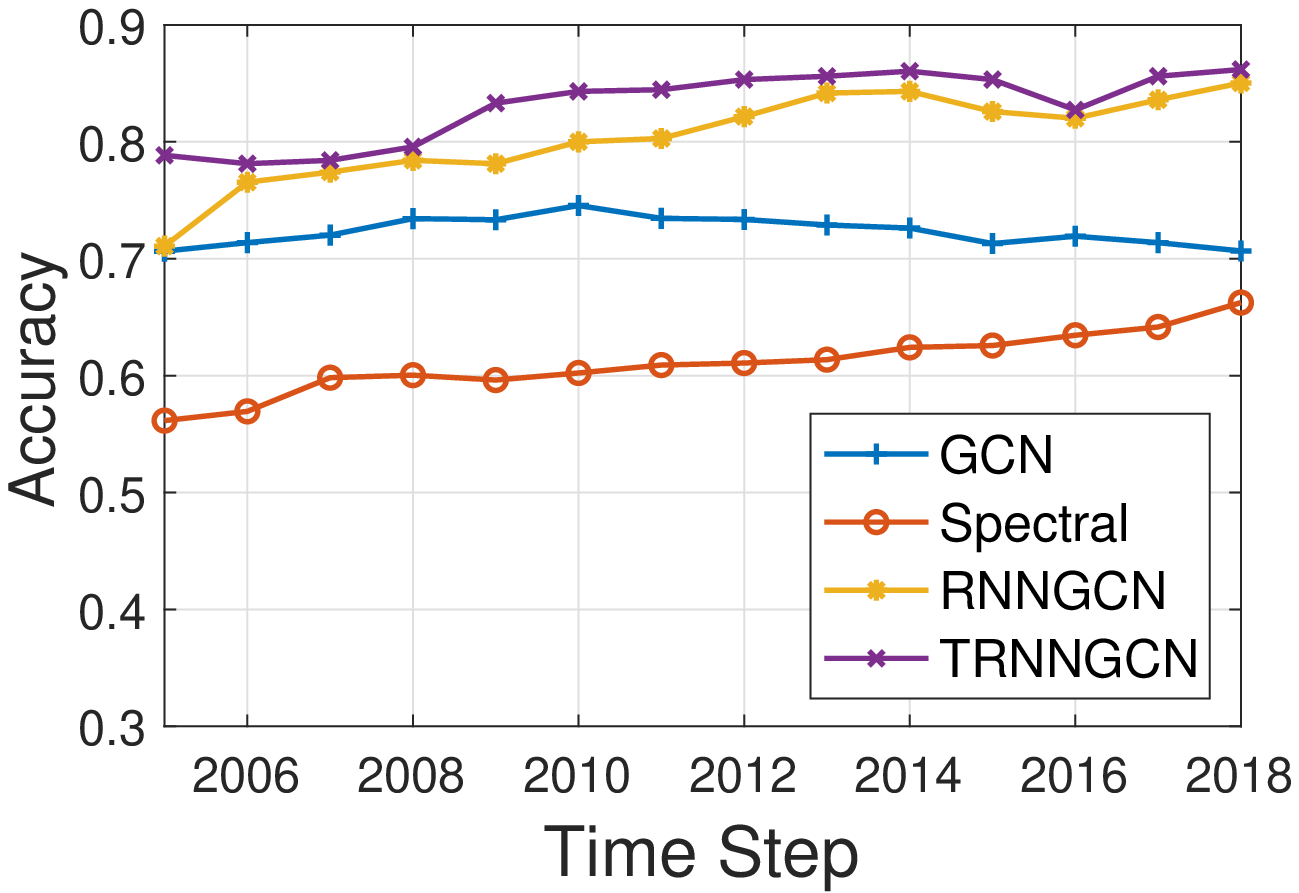}
    \includegraphics[width=0.23\textwidth,trim={0.25cm 0cm 0.8cm 0.5cm}, clip]{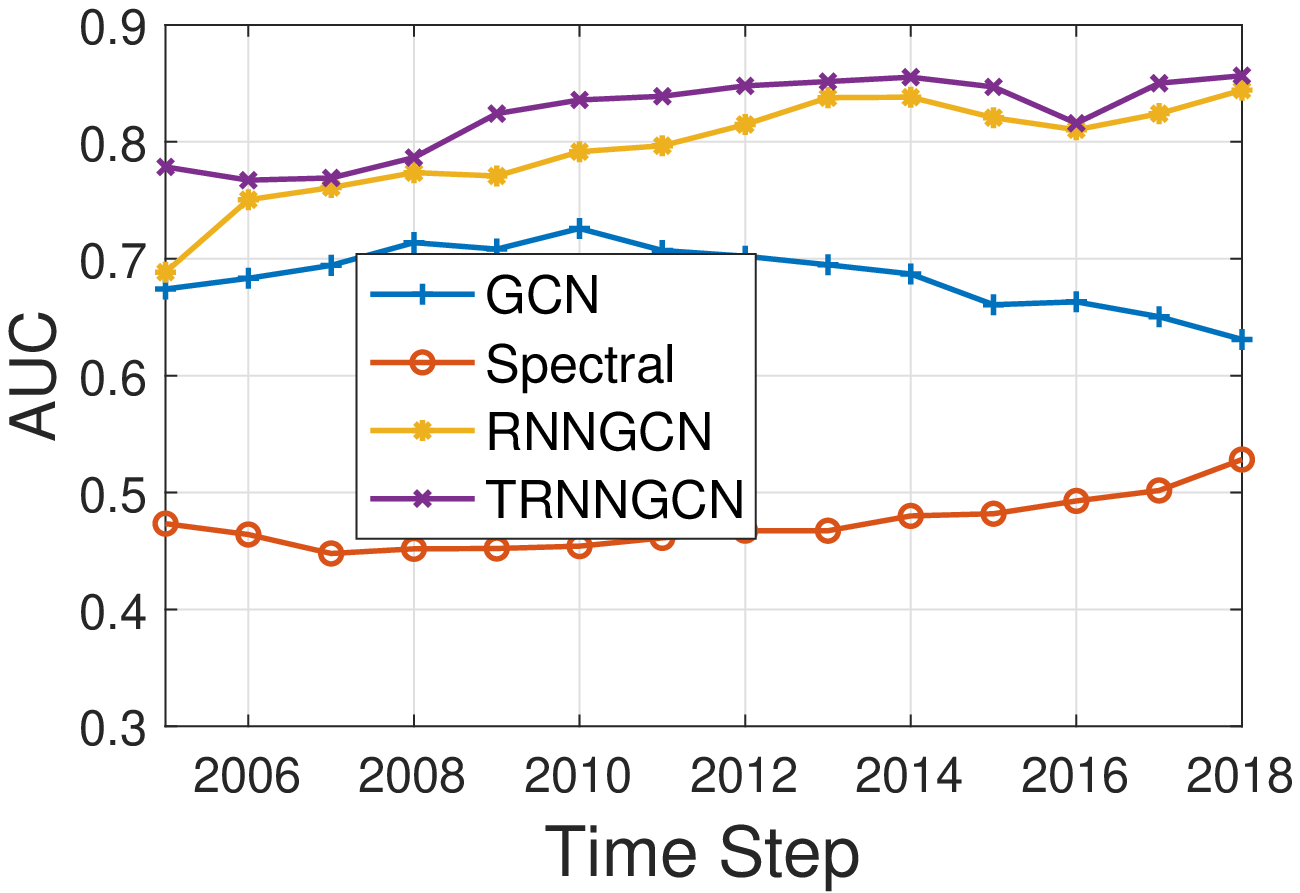}
    \includegraphics[width=0.23\textwidth,trim={0.25cm 0cm 0.8cm 0.5cm}, clip]{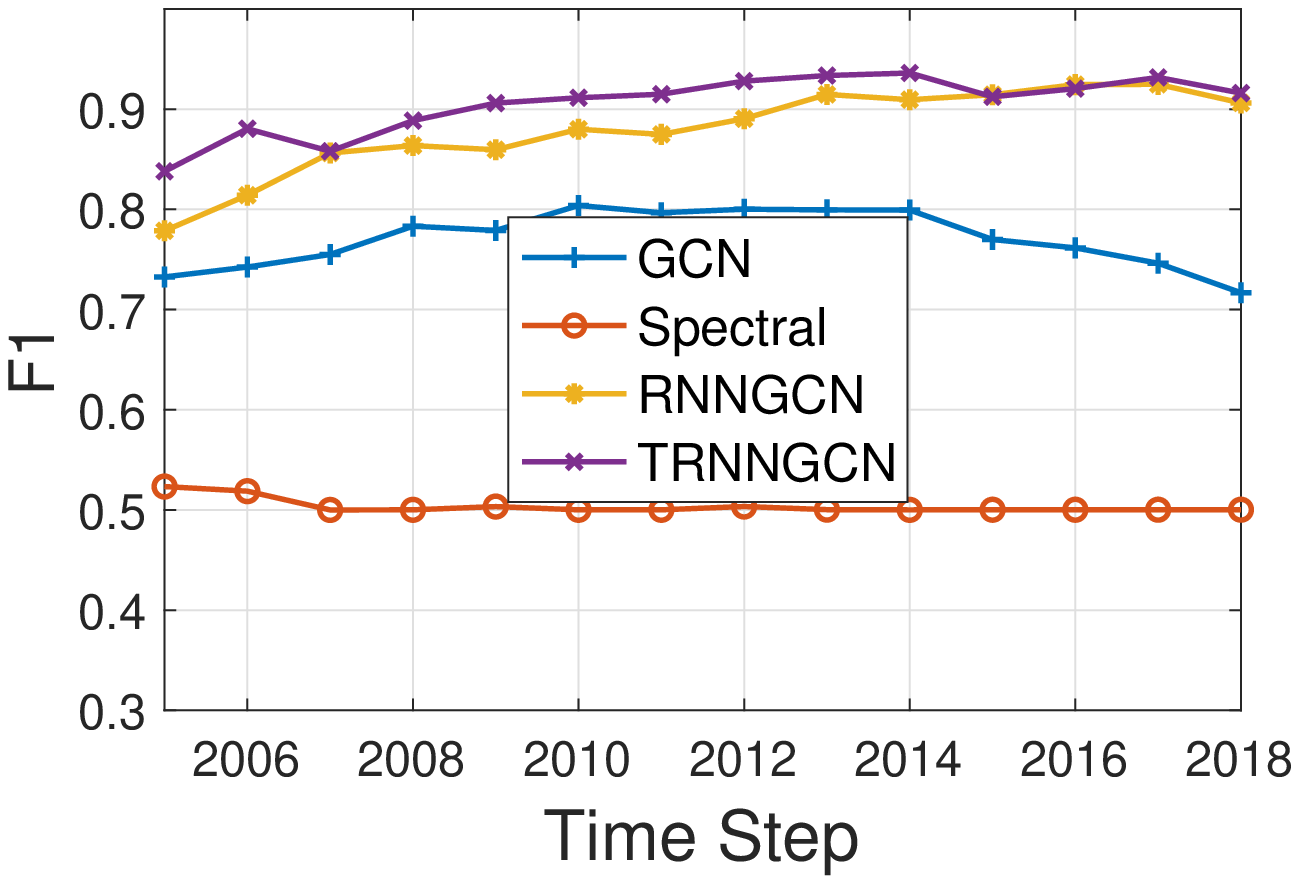}
    \includegraphics[width=0.23\textwidth,trim={0.25cm 0cm 1.2cm 0cm}, clip]{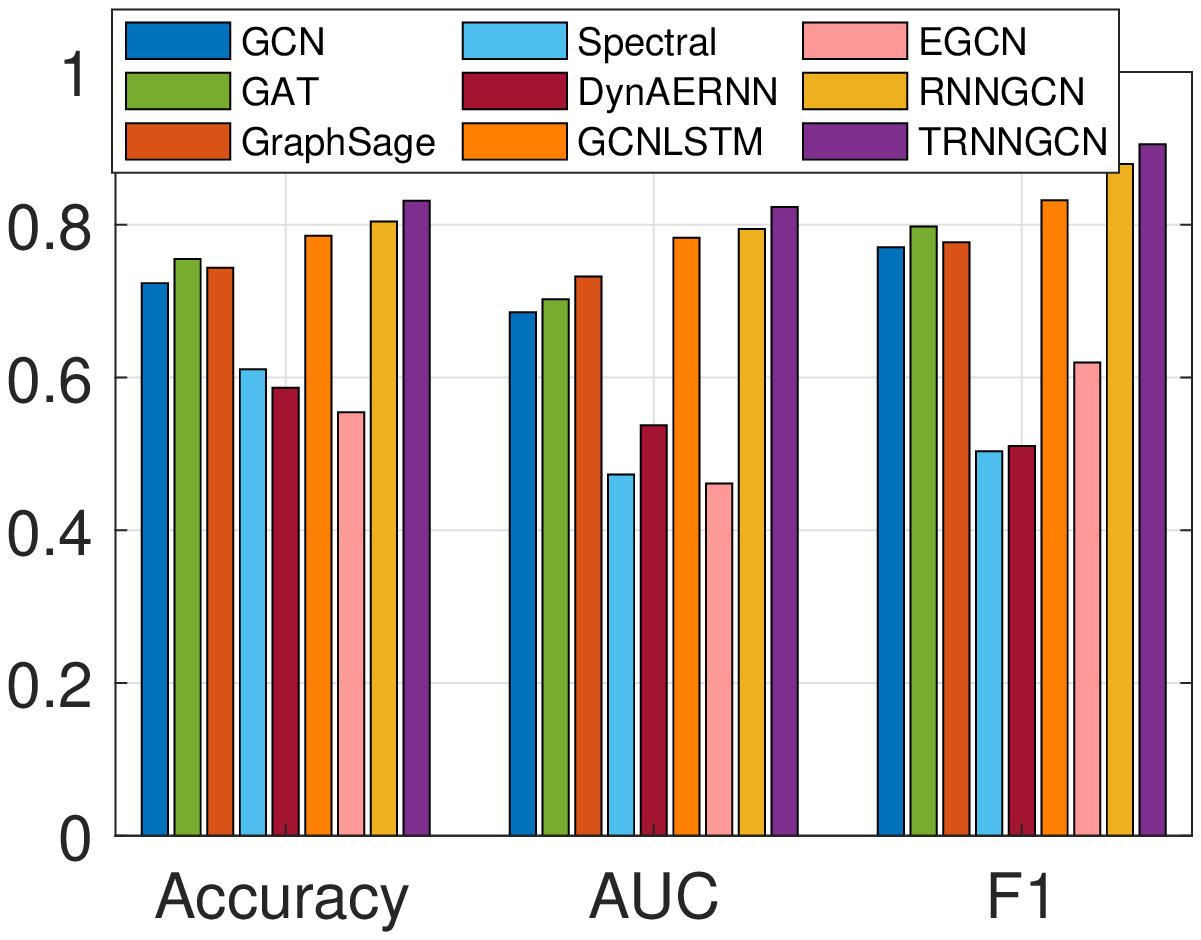}
    \caption{On DBLP-E data, TRNNGCN and RNNGCN outperform the static GCN and spectral clustering methods and show better performance over time, indicating that they can optimize their decay rates to account for historical information. TRNNGCN slightly outperforms RNNGCN.}
    \label{fig:dblpe}
\end{figure*}

\begin{table*}[ht]
    \centering
    
    \begin{tabular}{ccccccccccccc}
    \hline
         & \multicolumn{3}{c|}{DBLP-3} &\multicolumn{3}{c|}{DBLP-5}  &\multicolumn{3}{c|}{Reddit}  &\multicolumn{3}{c}{Brain}\\
    \hline
          & ACC&AUC&F1& ACC&AUC&F1&ACC&AUC&F1&ACC&AUC&F1 \\
    \hline
          GCN&71.6&62.2&35.8&64.9&51.0&\textbf{\textit{58.7}}&31.0&24.5&47.4&35.2&25.0&80.3\\
          GAT&70.9&59.4&57.8&62.3&48.2&51.4&16.8&4.8&50.0&34.6&26.4&81.6\\
          GraphSage &74.5&63.6&55.0&\textbf{\textit{66.5}}&53.9&55.1&29.2&20.7&42.5&\textbf{44.2}&\textbf{\textit{41.9}}&\textbf{86.7}\\
          Spectral&45.7&51.6&51.2&43.8&45.6&51.3&30.1&24.1&51.7&42.7&41.7&68.1\\
          DynAERNN &48.1&54.2&50.8&33.1&39.1&51.2&31.1&\textbf{31.7}&\textbf{54.1}&20.5&20.3&55.6\\
          GCNLSTM&74.5&63.6&48.4&\textbf{\textit{66.5}}&53.2&54.6&31.9&25.5&46.1&38.8&32.9&\textbf{\textit{85.9}}\\
          EGCN&72.3&60.7&48.1&63.2&50.6&53.2&28.3&12.5&50.0&28.6&26.1&73.7\\
          RNNGCN&\textbf{\textit{75.9}}&\textbf{\textit{68.0}}&\textbf{\textit{66.7}}&65.7&\textbf{\textit{55.4}}&58.6&\textbf{33.6}&20.5&49.7&41.0&38.6&84.7\\
          TRNNGCN&\textbf{78.0}&\textbf{72.1}&\textbf{73.8}&\textbf{67.4}&\textbf{57.9}&\textbf{63.5}&\textbf{33.6}&\textbf{\textit{25.6}}&\textbf{\textit{53.2}}&\textbf{\textit{43.8}}&\textbf{42.4}&85.7\\
    \hline
    \end{tabular}

    \caption{TRNNGCN consistently achieves the best (bold) or second-best (bold italics) accuracy (ACC), area under the ROC curve (AUC), and F1 score compared to baseline algorithms on four temporal datasets.}
    \label{tab:result_temporal}
\end{table*}

Figure \ref{fig:dblpe} shows the performance of RNNGCN and TRNNGCN in DBLP-E. Similar to Figure~\ref{fig:simulate}'s result on the simulated data, GCN's performance decreases over time as the accumulated effect of class change increases. Spectral Clustering consistently performs poorly since it cannot learn the high-dimensional patterns of the DBLP-E graph. RNNGCN and TRNNGCN maintain good performance and fully utilize the temporal information. As the two classes have different change probabilities, TRNNGCN learns a better decay rate and performs better. We further show the average accuracy, AUC, and F1-score for each baseline method at each timestep; RNNGCN and TRNNGCN consistently outperform the other baselines. GCNLSTM comes the closest to matching their performance. GCNLSTM uses a LSTM layer to account for historical information, which is similar to our methods but lacks interpretability as the LSTM operates on the output of the GCN layer (which is not readily interpretable) instead of the original graph adjacency information. We use a RNN instead of LSTM layer in our algorithms for computational efficiency.
%ACC AUC F! Column plot

\textbf{Node Classification with Temporal Features}
Although our analysis is based on dynamic networks without features, Table \ref{tab:result_temporal}'s performance results demonstrate the applicability of our RNNGCN and TRNNGCN algorithms to the DBLP-3, DBLP-5, Brain, and Reddit datasets with node features. TRNNGCN achieves the best or second-best performance consistently on all datasets, even outperforming EGCN and GCNLSTM, which unlike TRNNGCN fully utilize the historical information of node features. While GraphSAGE shows good accuracy, AUC, and F1-score on the Brain dataset, no other baseline method does well across all three metrics for any other dataset. RNNGCN performs second-best on DBLP-3 but worse on the other datasets, likely because those datasets have more than three classes, which would likely have different optimal decay rates. TRNNGCN can account for these differences, but RNNGCN cannot.

We further highlight the importance of taking into account historical information by noting that the static baselines (GCN, GAT, GraphSAGE, and spectral clustering) generally perform poorly compared to the dynamic baselines (DynAERNN, GCNLSTM, EGCN). DynAERNN can perform significantly worse than GCNLSTM and EGCN, likely because it is an unsupervised method that cannot take advantage of labeled training data. Thus, RNNGCN and TRNNGCN's good performance is likely due to their ability to optimally take advantage of historical graph information, even if they cannot use historical node feature information.

\section{Conclusion}

This work proposes RNNGCN and TRNNGCN, two new neural network architectures for clustering on dynamic graphs. These methods are inspired by the insight that RNNs progressively decrease the weight placed on their inputs over time according to a learned decay rate parameter. %They thus introduce a decay rate that weights connections between nodes as more time passes since the connection formation.
This decay rate can in turn be interpreted as the importance of historical connection information associated with each community or cluster in the graph.
We show that decaying historical connection information can achieve almost exact recovery when used for spectral clustering on dynamic stochastic block models, and that the RNN decay rates on simulated data match the theoretically optimal decay rates for such stochastic block models. We finally validate the performance of RNNGCN and TRNNGCN on a range of real datasets, showing that TRNNGCN consistently outperforms static clustering methods as well as previously proposed dynamic clustering methods. This performance is particularly remarkable compared to dynamic clustering methods that account for historical information of both the connections between nodes and the node features; TRNNGCN ignores the historical feature information. We plan to investigate neural network architectures that reveal the importance of these dynamic node features in our future work. Much work also remains on better establishing 
the models' interpretability. 

%to account for historical feature values in the clustering algorithm,

%Future work : temporal labels; theory analysis on weight decay method of node features. 

\newpage

%\section*{Ethical Impact}

%We do not anticipate this work having direct ethical impacts. Users may deploy our RNNGCN or TRNNGCN architectures to better cluster nodes in a graph, which may allow them to infer private information about the nodes in the graph. Conversely, one could use a better understanding of node clusters to help counteract attacks on vulnerable nodes. Our methods allow the user to understand the importance of historical information in the clustering, which may help them better explain the decisions of the clustering algorithm and evaluate if these decisions are appropriate.

\section*{Acknowledgements}
This research was partially supported by NSF grant CNS-1909306. The authors would like to thank Jinhang Zuo, Mengqiu Teng and Xiao Zeng for their inputs to the work.

\bibliography{ref}

\newpage

\appendix
\section{TRNNGCN}
 \begin{algorithm}[!h]
 \SetAlgoLined
 \SetKwInOut{Input}{Input}
 \SetKwInOut{Output}{Output}
 \Input{Temporal graph $(A_1,A_2,...,A_T)$,\\ Membership matrix of training data $\Theta_T^{train}$}
 \Output{Membership matrix estimate $\hat{\Theta}_T$}

  $\hat{A}=A_0$, $H_0=I_N$\;
 \For{\text{iteration} $i=1,...,I$}{
 \For{$ t=2,...,T$ }{
 $ \hat{A}_t=(1-\hat{\Theta}_{i-1} \Lambda (\hat{\Theta}_{i-1})^T) \circ \hat{A}_{t-1}+\hat{\Theta}_{i-1} \Lambda (\hat{\Theta}_{i-1})^T \circ \hat{A}_t.$}

 $H^{(1)}=\sigma_1(\hat{A}_TH^{(0)}W^{(1)})$\\
 $H^{(2)}=\sigma_2(\hat{A}_TH^{(1)}W^{(2)})$\\
 CrossEntropyLoss($H_i^{train}$, ${\Theta}_i^{train}$)\\
 Backward()\\
 $\hat{\Theta}_{i}=\text{Onehot}(\argmax_{1\leq j \leq n}H_{jk}^{(2)})$\\
 }
 $\hat{\Theta}_T=\text{Onehot}(\argmax_{1\leq j \leq n}H_{jk}^{(2)})$

 \caption{TRNNGCN}
\label{algo:TRNNGCN}
\end{algorithm}

\section{Recovery Requirements}
The goal of clustering or community detection is to recover the membership $\Theta$ by observing the graph $G$, up to some level of accuracy. We next define the relative error and rate of recovery.

\begin{definition}
[Relative error of $\hat{\Theta}$]
The relative error of a clustering estimate $\hat{\Theta}$ is
\begin{equation}
E(\hat{\Theta},\Theta)=\frac{1}{n} \min_{\pi\in \mathcal{P}} \|\hat{\Theta} \pi-\Theta\|_0,    
\end{equation}
where $n$ denotes the number of nodes in graph $G$, $\mathcal{P}$ is the set of all $K\times K$ permutation matrices and $\|.\|_0$ counts the number of non-zero elements of a matrix.
\end{definition}

\begin{definition}
[Agreement or accuracy of $\hat{\Theta}$]
\begin{equation}
A(\hat{\Theta},\Theta)=1-E(\hat{\Theta},\Theta).    
\end{equation}

\end{definition}

\begin{definition}{(Recovery rate of $\hat{\Theta}$)}
A clustering estimate $\hat{\Theta}$ achieves 

\begin{itemize}
\item Exact Recovery when
\begin{equation*}
    \mathbb{P}\{A(\hat{\Theta},\Theta)=1\}=1-o(1),
\end{equation*}
\item Almost Exact Recovery when 
\begin{equation*}
    \mathbb{P}\{A(\hat{\Theta},\Theta)=1-o(1)\}=1-o(1),
\end{equation*}

\item Partial Recovery when

\begin{equation*}
    \mathbb{P}\{A(\hat{\Theta},\Theta)\geq \alpha\}=1-o(1),\alpha\in(\frac{1}{k},1).
\end{equation*}
\end{itemize}
\end{definition}

\begin{lemma}
\label{lemma:almost}
If 

\begin{equation}
    \mathbb{E}(A(\hat{\Theta},\Theta))=1-o(1),
\end{equation}where $\mathbb{E}$ denotes the expectation, then $\hat{\Theta}$ achieves almost exact recovery.
\end{lemma}

\begin{lemma}
\label{lemma:partial}
If \begin{equation}
    \mathbb{E}(A(\hat{\Theta},\Theta))\geq \alpha-o(1), \alpha\in(\frac{1}{k},1),
\end{equation} then $\hat{\Theta}$ achieves partial recovery.
\end{lemma}

\section{Proof of Proposition 1}

\begin{lemma}
In \cite{abbe2017community}, by genie-aided hypothesis test, if the link probability $\alpha=\mathcal{O}(\frac{\log n}{n})$, at time step $t$, by spectral estimator, we have
$$\mathbb{P}(\hat{\Theta}_{t,i}=\Theta_{t,i})=\left\{
\begin{aligned}
1-o(1) & , & \frac t 2\geq \max {C_i}\\
o(1) & , & \frac t 2 < \max {C_i},
\end{aligned}
\right.$$
where $\Theta^t_i$ denotes the i-th row (membership of node $i$) of membership matrix $\Theta$ at timestep $t$  and $C_i$ denotes the most recent change time of node $i$.
\end{lemma}

\begin{prop}[Partial Recovery of Spectral Clustering]\label{prop:recover_rate_appendix}
When nodes change their cluster membership over time with change probabilities $\mathcal{O}(\frac{\log n}{n})$, Spectral Clustering recovers the true clusters at time $T$ with relative error $\mathcal{O}(\frac{\log n}{n} T)$. %Add proof in support material
\end{prop}

\begin{proof}

Let $X^c_t$ denote the set of changed nodes from timestep $1$ to timestep $t$, and $X^u_t$ denote the set of the remaining unchanged nodes. $\Theta^c_t$ means the membership matrix of $X^c_t$ and $\Theta^u_t$ is the membership matrix of $X^u_t$ . We have
\begin{equation*}
    \begin{aligned}
    &\mathbb{E}(A(\hat{\Theta}_t,\Theta_t))\\=&\mathbb{E}(A(\hat{\Theta}^u_{t}\cup \hat{\Theta}^c_{t},\Theta^u_{t}\cup \Theta^c_{t})=1)\\
    \leq&\sum_{i\in X^u_{t}} \mathbb{E}(\hat{\Theta}_{t,i}=\Theta_{t,i})+\sum_{j\in X^c_{t}} \mathbb{E}(\hat{\Theta}_{t,j}=\Theta_{t,j})\\
    =&\sum_{i\in X^u_{t}} \mathbb{E}(\hat{\Theta}_{t,i}=\Theta_{t,i})+\sum_{s=1}^{t}\sum_{j\in X^c_{s}/X^c_{s-1}} \mathbb{E}(\hat{\Theta}_{t,j}=\Theta_{t,j})\\
    =&(1-\mathcal{O}(\frac{\log n}{n}t))\mathbb{P}(\hat{\Theta}_{t,i}=\Theta_{t,i})_{i\in X^u_{t}}\\
    &+\mathcal{O}(\frac{\log n}{2n}t)\mathbb{P}(\hat{\Theta}_{t,j},\Theta_{t,j})_{j\in X^c_{t}, \frac{t}{2}\geq \max C_j}\\
    &+\mathcal{O}(\frac{\log n}{2n}t)\mathbb{P}(\hat{\Theta}_{t,j},\Theta_{t,j})_{j\in X^c_{t}, \frac{t}{2}< \max C_j}\\
    =&(1-\mathcal{O}(\frac{\log n}{n}t))(1-o(1))\\
    &+\mathcal{O}(\frac{\log n}{2n}t)(1-o(1))+\mathcal{O}(\frac{\log n}{2n}t)o(1)\\
    =& 1-\mathcal{O}(\frac{\log n}{n}t)
    \end{aligned}
\end{equation*}
Then the expectation of relative error at timestep $T$ is 
\begin{equation}
    \mathbb{E}(E(\hat{\Theta}_T,\Theta_T))= \mathcal{O}(\frac{\log n}{n} T).
\end{equation}
When $T=\mathcal{O}(\frac{n}{\log n})$, we have $\mathbb{E}(A(\hat{\Theta}_T,\Theta_T))=1-\mathcal{O}(1)$. By Lemma \ref{lemma:partial}(Technical Appendix), it only solves partial recovery.
\end{proof}

\section{Proof of Proposition 2}

\begin{prop}[Optimal Decay Rate]\label{prop:opt_appendix}
The concentration of each block $k$ is upper-bounded by
\begin{equation}
\label{equ:bound}
    \left\|\hat{A}_{t}^{k}-P_{t}^{k}\right\|\lesssim E_1(\beta_{k})+E_2(\beta_{k}),
\end{equation}
where $\beta_k$ denotes the max decay rate of class $k$ and
\begin{equation} %\left\{
%\begin{aligned}
E_1(\beta_{k}) = \sqrt{ n \alpha \beta_{k}},\;
E_2(\beta_{k}) = \alpha \sqrt{\frac{n^2 \varepsilon_k}{\beta_{k}}},
%\end{aligned}
%\right.
\end{equation}
which is minimized when $\beta_{k}=\sqrt{ n \alpha \varepsilon_k}$.
\end{prop}

\subsection{Decay rule}
In dynamic SBM, we use the following decay rule:
\begin{equation}
    \hat{A}_t=(1-\Theta_t \Lambda (\Theta_t)^T)\hat{A}_{t-1}+\Theta_t \Lambda (\Theta_t)^T A_t.
\end{equation}
Let $\hat{A}_{t}^{k}$ denote the block matrix corresponding to cluster $k$, and similarly consider $K$ blocks $P_t^k$ of the connection probability matrix $P_t$. Let $\lambda_k$ denote the $k$-th element in the diagonal of $\Lambda$. We have 
\begin{equation}
    \hat{A}_t^k=(1-\lambda_k)\hat{A}_{t-1}^k+\lambda_k A_t^k,
\end{equation}
which can also be written as
\begin{equation}
    \hat{A}_t^k=\sum_{s=0}^{t}\beta_s^k {A}_{t-s}^k,
\end{equation}
where $\beta_s^k=\lambda_k (1-\lambda_k)^s$ for $s<t$ and $\beta_{t}^k=(1-\lambda)^t$. Then we denote the maximum of $\beta_s^k$ as $\beta_k$ and we have $\beta_k=\lambda_k$.
Similarly, we define
\begin{equation}
    \hat{P}_t^k=\sum_{s=0}^{t}\beta_s^k {P}_{t-s}^k.
\end{equation}
\subsection{Error Bound}
The error bound $\left\|\hat{A}_{t}^{k}-P_{t}^{k}\right\|$ can be divided into two terms:
\begin{equation}
    \left\|\hat{A}_{t}^{k}-P_{t}^{k}\right\|\leq \left\|\hat{A}_{t}^{k}-\hat{P}_{t}^{k}\right\|+ \left\|\hat{P}_{t}^{k}-P_{t}^{k}\right\|
\end{equation}
\subsubsection{Preliminaries}
The result is valid for any estimator with weights $\beta_s^k\geq 0$ that satisfy the property that there are constants $C_\beta,C'_\beta>0$ such that:
\begin{equation}
\begin{aligned}
&\sum_{s=0}^t \beta_s^k =1,\;\; \beta_s^k \leq \beta_{k},\;\;\sum_{s=0}^t (\beta_s^k)^2\leq C_\beta \beta_{k} \\
&\sum_{s=0}^t \beta_s^k \min(1,\sqrt{s \varepsilon_k})\leq C'_\beta \sqrt{\frac{\varepsilon_k}{\beta_k}}
\end{aligned}
\end{equation}
Our defined decay rate naturally satisfy the first three preliminaries. The last preliminary is satisfied when $t\geq \frac{\min (\log(\varepsilon_k/\beta_k),\log \beta_k)}{2 \log (1-\beta_k)}$

\subsubsection{Bound the first term}
\begin{theorem}
\cite{keriven2020sparse} Let $A_1,...,A_t\in\{0,1\}^{n}$ be $t$ symmetric Bernoulli matrices whose elements $a^{s}_{ij}$ are independent random variables:
\begin{equation}
    a^{s}_{ij} \backsim Ber(p^{s}_{ij}), a^{s}_{ji}=a^{s}_{ij}, a^{s}_{ii}=0
\end{equation}
Assume $p^{s}_{ij}\leq \alpha$. Consider non-negative weights $\beta_s$ and $A=\sum_{s=0}^t \beta_s A_{t-s}$ and $P=\mathbb{E}(A)$, there is a universal constant $C$ such that for all $c>0$, we have 
\begin{equation}
\begin{aligned}
&\mathbb{P}\left(\left\|A-P\right\|\geq C(1+c)\sqrt{n \alpha \beta_{\max}}\right)\\
&\leq e^{-(\frac{c^2/2}{2C_\beta + 2c/3}-\log14)n}
    +e^{-\frac{c^2/2}{2C_\beta + 2c/3}\frac{n \alpha}{\beta_{\max}}+\log n}+n^{-\frac{c}{4}+6}.
\end{aligned}
\end{equation}
\label{theorem:first_term}
\end{theorem}
By applying Theorem \ref{theorem:first_term}, for fixed block $\Theta_0^k,...,\Theta_t^k$, if ${\frac{n \alpha}{\beta_k}}\gtrsim \log n$, then for any $\nu>0$, there is a constant $C_\nu$ such that with probability at least $1-n^\nu$
\begin{equation}
\begin{aligned}
\left\|\hat{A}_{t}^{k}-\hat{P}_{t}^{k}\right\|
    &\leq \left\|\hat{A}_{t}^{k}-\mathbb{E}(\hat{A}_{t}^{k})\right\|+\left\|diag(\hat{P}_t^k) \right\|\\
    &\leq C_\nu \sqrt{n\alpha \beta_k} + \alpha
\end{aligned}
\end{equation}

In all considered case, $\beta_k \gg \frac{1}{n}$, $\alpha$ is negligible, we have
\begin{equation}
    \left\| \hat{A}_t^k - \hat{P}_t^k\right\|\lesssim \sqrt{n \alpha \beta_k}
\end{equation}

\subsubsection{Bound the second term}
Since $\sum_{s=0}^t \beta_s^k=1$, we have
\begin{equation}
    \left\| \hat{P}_t^k - P_t^k\right\|\leq \sum_{s=0}^t \beta_s^k \left\|P_{t-s}^k-P_t^k \right\|\leq \sum_{s=0}^t \beta_s^k \left\|P_{t-s}^k-P_t^k \right\|_F,
\end{equation}
where $\|.\|_F$ is the Frobenius norm. 

\begin{lemma}
\label{lemma:bound_p}
Consider $P=\Theta B \Theta^T$ and $P'=\Theta' B (\Theta')^T$. Let $S$ denotes the set of nodes that have changed clusters between $\Theta$ and $\Theta'$. Then we have 
\begin{equation}
\begin{aligned}
\left\| P-P'\right\|^2_F &=\sum_{i\in S}\sum_j(p_{ij}-p'_{ij})+(p_{ji}-p'_{ji}) \\
&\leq 4 \sum_{i\in S} \sum_j p_{ij}^2 + (p'_{ij})^2 \\
& \leq 8 \alpha^2 |S| n \max_k \sum_l (B)^2_{kl} \\
&\leq 8 \alpha^2 |S| n
\end{aligned}
\end{equation}
\end{lemma}
By using Lemma \ref{lemma:bound_p} and the Lemma \cite{keriven2020sparse}
\begin{equation}
\begin{aligned}
\mathbb{P}( \exists k,\; \left\|P_{t-k}-P_t\right\|^2_F  \geq  (8+&C) \alpha n^2 \min (1,k\varepsilon) )\\
&\leq e^{-2C^2\varepsilon^2n+\log \frac{1}{\varepsilon}}
\end{aligned}
\end{equation}
we have, with probability at least $1-n^{\nu}$,
\begin{equation}
    \left\|P_{t-s}^k-P_t^k \right\|_F\leq 2 \alpha n^2 \min (1,k\varepsilon_k))
\end{equation}
By the preliminary of $\sum_{s=0}^t \beta_s^k \min(1,\sqrt{s \varepsilon_k})$, we obtain the desired bound.
\begin{equation}
    \left\| \hat{P}_t^k - P_t^k\right\|\lesssim \alpha \sqrt{\frac{n^2 \varepsilon_k}{\beta_{k}}}
\end{equation}

\section{Proof of Proposition 3}

\begin{prop}[Almost Exact Recovery]\label{prop:bound_appendix}
 Let $\lambda_{\max}$ denote the maximum element on the diagonal of $\Lambda$. With probability at least $1-n^{-\nu}$ for any $\nu > 0$, at any time $t$ we have 
\begin{equation}
    \left\|\hat{A}_t-P_t\right\|\lesssim \sqrt{ n \alpha \lambda_{\max}}
\end{equation}
When $K$ is constant, $\varepsilon_k=\mathcal{O}(\frac{\log n}{n})$ and $\alpha=\mathcal{O}(\frac{\log n}{n})$, the relative error is $\mathcal{O}\left(\frac{1}{n^{\frac{1}{4}}\log n}\right)$, which implies almost exact recovery at time $T$. 
\end{prop}

\begin{proof}
By Equation \ref{equ:bound}(Technical Appendix), we have
\begin{equation}
    \left\|\hat{A}_{t}^{k}-P_{t}^{k}\right\| \lesssim \sqrt{ n \alpha \lambda_{k}}.
\end{equation}

We define the decay rates as
\begin{equation}
    \Lambda_{jk}=\left\{
\begin{aligned}
\min(1,\sqrt{n \alpha  \varepsilon_k})& , &j=k\\
1& , &j\neq k.
\end{aligned}
\right.
\end{equation}

We can separate $\hat{A}_t$ into $K^2$ blocks based on their belonging of clusters. Let $\hat{A}_t^d$ denote the matrix include the diagonal block of $\hat{A}_t$ and $\hat{A}_t^n$ denote the matrix include the non-diagonal block of $\hat{A}_t$. The same goes in $\hat{A}_t^d$ and $\hat{P}_t^n$. We have

\begin{equation}
    \left\|\hat{A}_t-P_t\right\|\leq \left\|\hat{A}_t^d-P_t^d\right\| + \left\|\hat{A}_t^n-P_t^n\right\|
\end{equation}

For $\left\|\hat{A}_t^n-P_t^n\right\|$, since $\tau \ll 1$, by setting $\Lambda_{jk}=1, (j\neq k)$, spectral norms of the matrix is negligible. Then we have
\begin{equation}
    \left\|\hat{A}_t-P_t\right\|\lesssim \left\|\hat{A}_t^d-P_t^d\right\| 
\end{equation}

Since $K$ is constant, by the property of spectral norm, we have
\begin{equation}
    \left\|\hat{A}_t-P_t\right\|\lesssim \max_k{\left\|\hat{A}_{t}^{k}-P_{t}^{k}\right\|} =\sqrt{ n \alpha \lambda_{\max}}
\end{equation}

When $\varepsilon_k=\mathcal{O}(\frac{\log n}{n})$ and $\alpha=\mathcal{O}(\frac{\log n}{n})$,
\begin{equation}
\begin{aligned}
&E(\hat{\Theta},\Theta)\\
\lesssim &(1+\delta)\frac{n_{\max}' K}{n\alpha^2 n_{\min}^2 \tau^2}\|\hat{A}-P\|^2\\
\lesssim & (1+\delta)\frac{n K}{n\alpha^2 n^2 \tau^2} \sqrt{ n \alpha \lambda_{\max}}\\
= & (1+\delta)\frac{K}{\mathcal{O}(\frac{\log n}{n})^2 n^2 \tau^2} \sqrt{ n \mathcal{O}(\frac{\log n}{n}) \sqrt{n \mathcal{O}(\frac{\log n}{n}) \mathcal{O}(\frac{\log n}{n})}}\\
=&(1+\delta)\frac{K}{\mathcal{O}({\log n})^2 \tau^2} \sqrt{ \mathcal{O}({\log n})^2 \sqrt{ \mathcal{O}(\frac{1}{n}) }}\\
=&(1+\delta)\frac{K}{ \tau^2} \mathcal{O}\left(\frac{1}{n^{\frac{1}{4}} \log n}\right).
\end{aligned}
\end{equation}
The relative error is $\mathcal{O}\left(\frac{1}{n^{\frac{1}{4}}\log n}\right)$ By Lemma \ref{lemma:almost}(Technical Appendix), it implies almost exact recovery at time $T$.
\end{proof}

\section{GCN and Spectral Clustering}
Figure \ref{fig:spec_norm_appendix} shows that Spectral Clustering and GCN have qualitatively similar accuracy on simulated data as we vary the decay rate $\lambda$ (the same $\lambda$ is used for all nodes). As expected from Proposition~\ref{prop:opt_appendix}, the optimal decay rate $\lambda$ increases as we increase the link probability $\alpha$, as does the value of $\lambda$ that minimizes the spectral norm $\|\hat{A} - P\|$. The optimal decay rate for GCN matches the decay rate value that minimizes the spectral norm, while the optimal decay rate for spectral clustering is larger than the one minimizing the spectral norm.
\begin{figure}[ht]
    \centering
    \includegraphics[width=0.23\textwidth,trim={0.3cm 0 0.8cm 0.3cm}, clip]{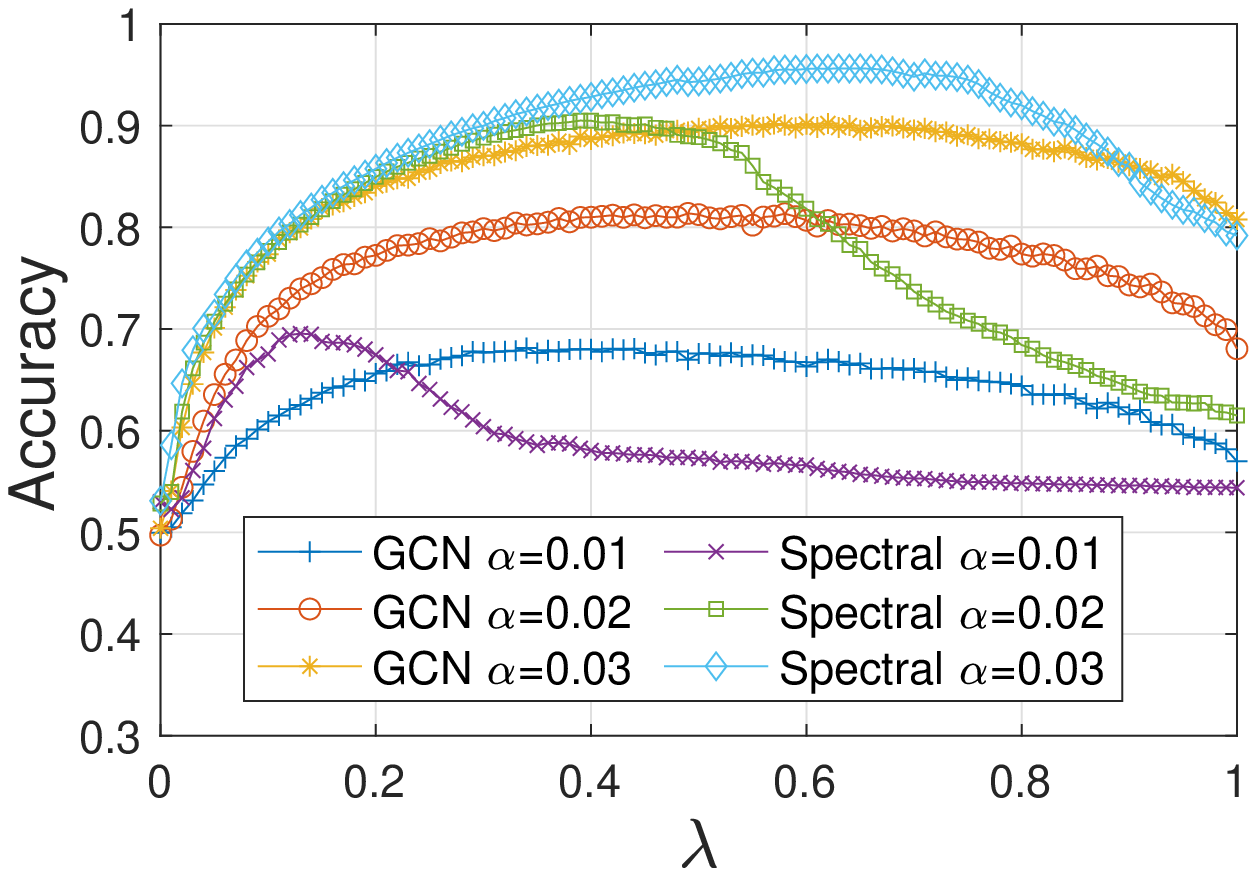}
    \includegraphics[width=0.23\textwidth,trim={0.3cm 0 0.8cm 0.3cm}, clip]{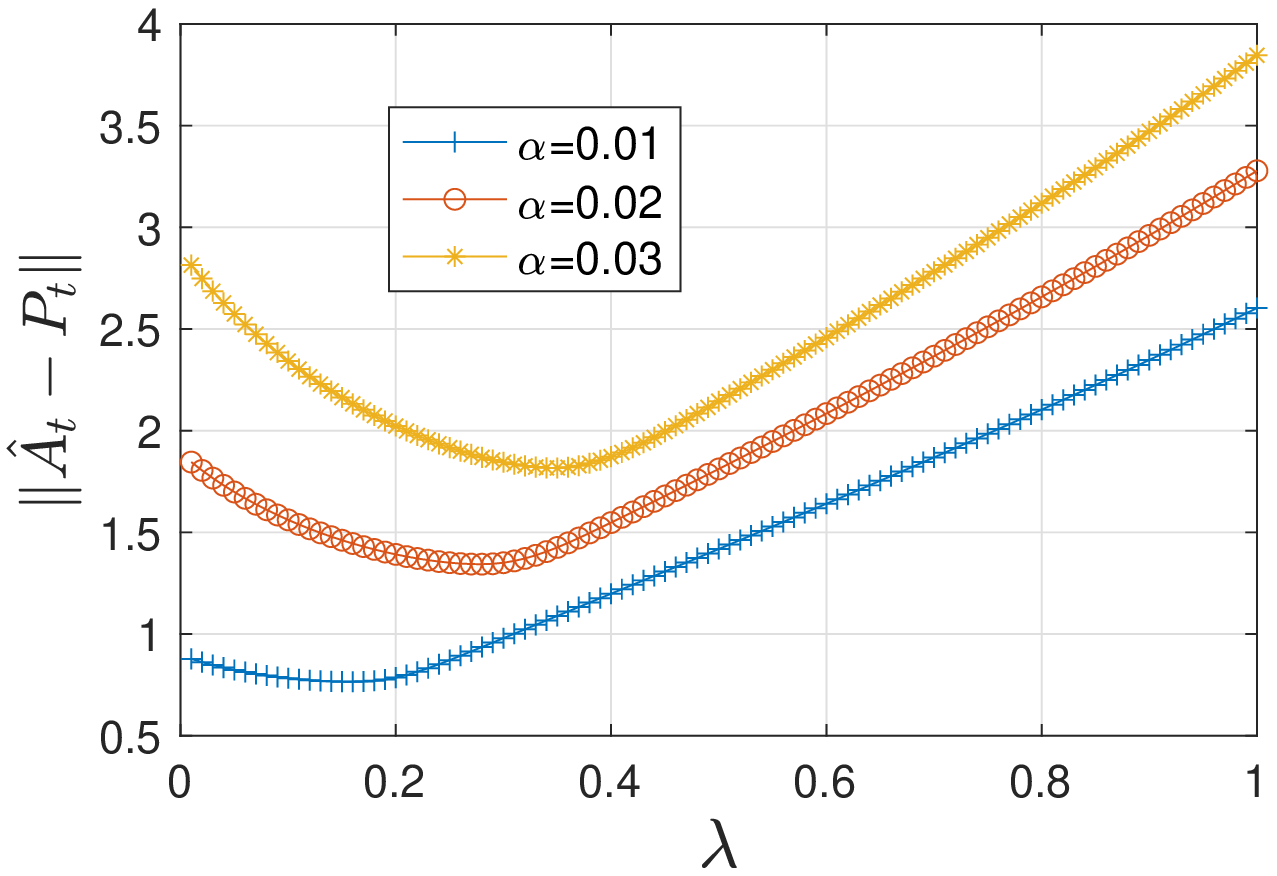}
    \caption{Accuracy and Spectral Norm as we vary $\alpha$. The optimal decay rate $\lambda$ increases with $\alpha$, as in Proposition~\ref{prop:opt_appendix}.}
    \label{fig:spec_norm_appendix}
\end{figure}
\section{Experiment result on simulated data}
Figure \ref{fig:bar_simu} shows the accuracy, AUC, and F1 score comparison of all baseline methods on simulated data, averaged over all 50 timesteps. Our TRNNGCN and RNNGCN methods show the best performance across all three metrics.

\begin{figure}[ht]
    \centering
    \includegraphics[width=0.47\textwidth]{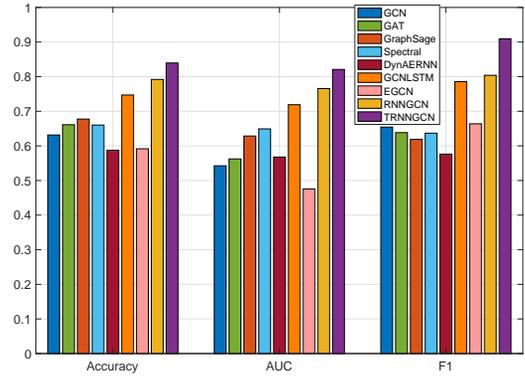}
    \caption{Comparison of methods on simulated data with heterogeneous cluster transition probabilities.}
    \label{fig:bar_simu}
\end{figure}

\end{document}